\documentclass{article}
\usepackage{makecell}
\usepackage{bbm}
\usepackage{cite}
\usepackage{epsfig}
\usepackage{latexsym}
\usepackage{setspace}
\usepackage{flushend}
\usepackage{authblk}
\usepackage{float}
\usepackage{dsfont}
\usepackage{comment}
\usepackage{ upgreek }
\usepackage{caption}
\usepackage{stmaryrd}
\usepackage{xspace}
\usepackage{url}
\usepackage{amsmath,amssymb,amsfonts,amsthm,color,xspace}
\usepackage{graphicx,tikz,enumitem}
\usepackage{verbatim}

\usepackage[utf8]{inputenc}
\usepackage[utf8]{inputenc} % allow utf-8 input
\usepackage[T1]{fontenc}    % use 8-bit T1 fonts
\usepackage{url}            % simple URL typesetting      % 
\usepackage{amsfonts}       % blackboard math symbols
\usepackage{nicefrac}       % compact symbols for 1/2, etc.
\usepackage{booktabs}
\usepackage{microtype} 
\usepackage[preprint,nonatbib]{neurips_2019}
\RequirePackage{algorithm}
\RequirePackage{algorithmic}
\newcommand{\ba}{\begin{align}}
\newcommand{\ea}{\end{align}}
\newcommand{\ban}{\begin{align*}}
\newcommand{\ean}{\end{align*}}
\newcommand{\lbl}[1]{\label{}}

\newcommand{\br}{\begin{remark}}
\newcommand{\er}{\end{remark}}

\makeatletter
\newtheorem*{rep@theorem}{\rep@title}
\newcommand{\newreptheorem}[2]{%
\newenvironment{rep#1}[1]{%
 \def\rep@title{#2 \ref{##1}}%
 \begin{rep@theorem}}%
 {\end{rep@theorem}}}
\makeatother

\newtheorem{theorem}{Theorem}

\newtheorem{corollary}[theorem]{Corollary}
\newtheorem*{corollary*}{Corollary}

\newtheorem*{observation*}{Observation}

\newtheorem{lemma}[theorem]{Lemma}
\newtheorem*{lemma*}{Lemma}

\newreptheorem{theorem}{Theorem}
\newreptheorem{lemma}{Lemma}

\theoremstyle{remark}
\newtheorem{remark}{Remark}
\newtheorem*{remark*}{Remark}
\newtheorem*{remarks*}{Remarks}

\theoremstyle{definition}

\def\undertilde#1{\mathord{\vtop{\ialign{##\crcr
$\hfil\displaystyle{#1}\hfil$\crcr\noalign{\kern1.5pt\nointerlineskip}
$\hfil\tilde{}\hfil$\crcr\noalign{\kern1.5pt}}}}}

\allowdisplaybreaks
\newcommand{\Sigmafield}{\Sigma}
\newcommand{\Samplespace}{\Omega}
\newcommand{\btotal}{m}

\newcommand{\alphabet}{{\cF}}
\newcommand{\bsize}{n}

\newcommand{\advfrac}{\beta}

\newcommand{\allbatches}{{B}}
\newcommand{\bgood}{{\allbatches_G}}
\newcommand{\badv}{\allbatches_A}

\newcommand{\Bsc}{\allbatches^\prime}
\newcommand{\Gsc}{{\allbatches^\prime_G}}
\newcommand{\Asc}{\allbatches^\prime_A}

\newcommand{\size}[1]{|{#1}|}

\newcommand{\bgoodsize}{\size{\bgood}}

\newcommand{\alphsubset}{S}
\newcommand{\alphsubsetdef}{ \alphsubset \in \alphabet}

\newcommand{\probestbatch}{\targetdis_b}

\newcommand{\targetdis}{p}
\newcommand{\batchdis}{p_b}
\newcommand{\targetdistance}{\eta}

\newcommand{\medV}{\text{med}}
\newcommand{\med}{\text{med}(\bar{\mu}(\alphsubset))}

\newcommand{\subsetprobtarget}{\targetdis(\alphsubset)}
\newcommand{\Uempprob}{{\bar{\targetdis}_{\Bsc}(\alphsubset)}}
\newcommand{\UGempprob}{{\bar{\targetdis}_{\Gsc}(\alphsubset)}}

\newcommand{\bempprob}{{\bar{\mu}_{b}(\alphsubset)}}

\newcommand{\VUempprob}{{\bar{\targetdis}_{\Bsc}}}

\newcommand{\Vbempprob}{{\bar{\mu}_{b}}} %previously  probestbatchsub

\newcommand{\empprob}[1]{\bar{\targetdis}_{#1}}

\newcommand{\empvarsub}[1]{\overline{\text{V}}_{#1}(\alphsubset)}

\newcommand{\var}[1]{\text{V}(#1)}

\newcommand{\corruption}{\psi}
\newcommand{\indcorruption}{\corruption_b(\alphsubset)}
\newcommand{\genoverallcorruptionsub}{\corruption_{\Bsc}(\alphsubset)}

\newtheorem{Theorem*}{Theorem}

\newtheorem{Claim*}[Theorem]{Claim}

\newtheorem{CounterExample*}{$\overline{\hbox{\bf Example}}$}

\newtheorem{Example*}[Theorem]{Example}

\newtheorem{Intuition*}[Theorem]{Intuition}
\newtheorem{Joke*}[Theorem]{Joke}

\newtheorem{Lemma*}[Theorem]{Lemma}
\newtheorem{Open problem}[Theorem]{Open problem}

\newtheorem{Question*}[Theorem]{Question}

\newcommand{\ignore}[1]{}

\newcommand{\RR}{\mathbb{R}}

\def \ba     {{\bf a}}

\newcommand{\cC}{{\mathcal C}}

\newcommand{\cF}{{\mathcal F}}

\newcommand{\cI}{{\mathcal I}}

\newcommand{\cO}{{\mathcal O}}
\newcommand{\cP}{{\mathcal P}}

\newcommand{\cS}{{\mathcal S}}

\newcommand{\cY}{{\mathcal Y}}

\newcommand{\reals}{\RR}
\newcommand{\eg}{\textit{e.g.,}\xspace}
\definecolor{light}{gray}{.75}

\def \upto  {{,}\ldots{,}}

\def \sets#1{{\{#1\}}}

\def \Paren#1{{\left({#1}\right)}}

\newcommand{\ed}{\stackrel{\mathrm{def}}{=}}

\def \half    {{\frac12}}

\def\ignore#1{}

\newcommand{\bi}{\begin{itemize}}
\newcommand{\ei}{\end{itemize}}
\def\orpro{\mathop{\mathchoice
   {\vee\kern-.49em\raise.7ex\hbox{$\cdot$}\kern.4em}
   {\vee\kern-.45em\raise.63ex\hbox{$\cdot$}\kern.2em}
   {\vee\kern-.4em\raise.3ex\hbox{$\cdot$}\kern.1em}
   {\vee\kern-.35em\raise2.2ex\hbox{$\cdot$}\kern.1em}}\limits}

\def\andpro{\mathop{\mathchoice
 {\wedge\kern-.46em\lower.69ex\hbox{$\cdot$}\kern.3em}
 {\wedge\kern-.46em\lower.58ex\hbox{$\cdot$}\kern.25em}
 {\wedge\kern-.38em\lower.5ex\hbox{$\cdot$}\kern.1em}
 {\wedge\kern-.3em\lower.5ex\hbox{$\cdot$}\kern.1em}}\limits}

\def\simge{\mathrel{%
   \rlap{\raise 0.511ex \hbox{$>$}}{\lower 0.511ex \hbox{$\sim$}}}}

\def\simle{\mathrel{
   \rlap{\raise 0.511ex \hbox{$<$}}{\lower 0.511ex \hbox{$\sim$}}}}

\usepackage{hyperref}
\hypersetup{colorlinks,linkcolor={blue},citecolor={blue},urlcolor={red}} 
\usepackage{selectp}
\makeatletter
\def\namedlabel#1#2{\begingroup
   \def\@currentlabel{#2}%
   \label{#1}\endgroup
}
\makeatother

\setlist[enumerate]{leftmargin=5.5mm}
\title{\mbox{A General Method for Robust Learning from Batches}}
\begin{document}
\author{
Ayush Jain and Alon Orlitsky\\
  University of California, San Diego\\
  \texttt{\{ayjain,alon\}@eng.ucsd.edu}
}
\maketitle

\begin{abstract}
In many applications, data is collected in batches, some of which are
corrupt or even adversarial. Recent work derived optimal robust algorithms for estimating discrete distributions in this setting.
We consider a general framework of robust learning from batches, and 
determine the limits of both classification and distribution estimation over arbitrary, including continuous, domains.
Building on these results, we derive the first robust agnostic computationally-efficient learning algorithms for piecewise-interval classification, and for piecewise-polynomial, monotone, log-concave, and gaussian-mixture distribution estimation. 
\end{abstract}

\section{Introduction}

%Huber In recent years, the robust approach was extended from learning to more general distribution learning. 
%Gaussians
%For general distributions there isn't - example
%Batches 

\subsection{Motivation}
In many learning applications, some samples are inadvertently or maliciously corrupted. A simple and intuitive example shows that this erroneous data limits the extent to which a distribution can be learned, even with infinitely many samples. Consider $p$ that could be one of two possible binary distributions: $(1,0)$ and $(1-\advfrac,\advfrac)$. Given any number of samples from $p$, an adversary who observes a $1-\beta$ fraction of the samples and can determine the rest, could use the observed samples to learn $p$, and set the remaining samples to make the distribution always appear to be $(1-\advfrac,\advfrac)$. Even with arbitrarily many samples, any 
estimator for $\targetdis$ fails to decide which $p$ is in effect, hence incurs a \emph{total-variation (TV)} distance $\ge \advfrac/2$, that 
we call the \emph{adversarial lower bound}.

The example may seem to suggest the pessimistic conclusion that if an adversary can corrupt a $\advfrac$ fraction of the data, a TV-loss of $\ge\advfrac/2$ is inevitable. Fortunately, that is not necessarily so.

In the following applications, and many others, data is collected in batches, most of which are genuine, but some possibly corrupted.
Data may be gathered by sensors, each providing a large amount of data, and some sensors may be faulty. The word frequency of an author may be estimated from several large texts,  some of which are mis-attributed. Or user preferences may be learned by querying several users, but some users may intentionally bias their feedback.
Interestingly, for data arriving in batches, even when a $\advfrac$-fraction of which are corrupted, more can be said. 

Recently,~\cite{qiao2017learning} formalized the problem for discrete domains. They considered estimating a distribution $\targetdis$ over $[k]$
in TV-distance when the samples are provided in batches of size $\ge\bsize$. A total of $\btotal$ batches are provided, of which a fraction $\le\beta$ may be arbitrarily and adversarially corrupted, while in every other batch $b$ the samples are drawn according a distribution $\batchdis$ satisfying  $||\batchdis-\targetdis||_{TV}\le\eta$, allowing for the possibility that slightly different distributions generate samples in each batch.

For $\advfrac\!<\!1/900$, they derived an estimation
algorithm that approximates any $\targetdis$ over a discrete domain 
to TV-distance $\epsilon =\cO(\eta+\advfrac/\sqrt{\bsize})$,  surprisingly, much lower than the individual samples limit of 
$\Theta(\eta+\advfrac)$.
They also derived a matching lower bound, showing that even for binary distributions, for any number $\btotal$ of batches, and hence for general discrete distributions, the lowest achievable total variation distance is $\ge \eta+ \frac\advfrac{2\sqrt{2\bsize}}$. We refer to this result as the \emph{adversarial batch lower bound}.

Their estimator requires $\cO(\frac{\bsize+k}{n\cdot\epsilon^2})$ batches of samples, or equivalently $\cO(\frac{\bsize+k}{\epsilon^2})$ samples in total, which is not always optimal. It also runs in time exponential in the domain size, rendering it impractical. 

Recently,~\cite{chen2019efficiently} reduced the exponential time complexity. 
Allowing quasi-polymoially many samples, they derived an estimator that achieves TV distance $\epsilon =\cO(\eta+\advfrac\sqrt{(\ln1/\advfrac)/\bsize)}$ and runs in 
quasi-polynomial time.
When a sufficiently larger distance is permitted, their estimator 
has polynomial time and sample complexities.
Concurrently,~\cite{jain2019robust} derived a polynomial-time, hence computationally efficient, estimator, that 
achieves the same 
$\cO(\eta+\advfrac\sqrt{(\ln1/\advfrac)/\bsize)}$ 
TV distance, and for domain size $k$
uses the optimal $\cO(k/\epsilon^2)$ samples. 

When learning general distributions in TV-distance, the sample complexity's linear dependence on the domain size is inevitable even when all samples are genuine. Hence, learning general distributions over large discrete, let alone continuous domains, is infeasible. 
To circumvent this difficulty,~\cite{chen2019efficiently} considered robust batch learning of structured discrete distributions, and 
studied the class of $t$-piecewise degree-$d$ polynomials over  the discrete set $[k]=\sets{1\upto k}$. 

They first reduced the noise with respect to an $\cF_k$ distance described later, and used existing methods on this cleaned data to estimate the distribution. This allowed them to construct an estimator that approximates these distributions with number of batches $m$ that grows only poly-logarithmically in the domain size $k$. Yet this number still grows with $k$, and is quasi-polynomial in other parameters $t$, $d$, batch size $n$, and $1/\advfrac$. Additionally, its computational complexity is quasi-polynomial in these parameters and the domain size $k$. 
Part of our paper generalizes and improves this technique. 

%To do so, they applied  $\cF_{k'}$-distance robustly in the batch setting. But, their estimator works only for discrete the domains and need number of batches $m$ quasi-polynomial in $k'$ and batch size $n$ and poly-logarithmic in the domain size $k$ and has computational complexity quasi-polynomial in all parameters $k'$, domain size $k$ and batch size $\bsize$.  

The above results suffer several setbacks. 
While for general distributions there are sample-optimal polynomial-time algorithms, for structured distributions 
existing algorithms have suboptimal quasi-polynomial sample 
and time complexity.
Furthermore both their sample- and time-complexities grow
to infinity in the domain size, making them impractical 
for many complex applications, and essentially impossible 
for the many practical applications with continuous domains such as $\reals$ or $\reals^d$. 

This leaves several natural questions.
For sample efficiency, can distributions over non-discrete spaces, 
be estimated in to the adversarial batch lower bound using finitely many samples, and if so, what is their sample complexity? 
For computational efficiency, are there estimators whose computational complexity is independent of the domain size, and can their run time be polynomial rather than quasi-polynomial in the other parameters. 
More broadly, can similar robustness results be derived 
for other important learning scenarios, such as classification?
And most importantly, is there a more general theory of robust learning from batches?

\subsection{Summary of techniques and contributions}
To answer these questions, we first briefly foray into VC theory.
Consider estimation of an unknown \emph{target distribution} $\targetdis$ to a small $\cF$-distance, where $\cF$ is a family of subsets with finite VC-dimension. %$V_{\cF}$. 
Without adversarial batches, the empirical distribution of samples from $p$ estimates it to a small $\cF$-distance. When some of the batches are adversarial, the empirical distribution could be far from $\targetdis$.
We construct an algorithm that "cleans" the batches and returns a sub-collection of batches whose empirical distribution approximates $p$ to near optimal $\cF$-distance. 

While the algorithm is near sample optimal, as expected from the setting's 
broad generality, for some subset families, the it is necessarily not computationally efficient.
We then consider the natural and important family $\cF_k$ of all unions of at most $k$ intervals in $\reals$. We provide a computationally efficient 
algorithm that estimates distributions to near-optimal $\cF_k$ distance
and requires only a small factor more samples than the best possible. 

Building on these techniques, we return to estimation in total variation (TV) distance. 
We consider the family of distributions whose Yatracos Class~\cite{yatracos1985rates} has finite VC dimension. 
This family consists of both discrete and continuous distributions, and includes piecewise polynomials, Gaussians in one or more dimensions, and
arguably most practical distribution families.
We provide a nearly-tight upper bound on the TV-distance 
to which these distributions can be learned robustly from batches.
%\tcr{We show that the number of samples required to achieve this 
%distance is similar to a log factor to the sample complexity
%upper bounds given by this general method for non-adversarial settings.}

Here too, the algorithms' broad generality makes them computationally inefficient some distribution classes. 
For one-dimensional $t$-piecewise degree-$d$ polynomials, we derive a
polynomial-time algorithm whose sample complexity has optimal linear dependence on $td$ and moderate dependence on other parameters. 
This is the first efficient algorithm for robust learning
of general continuous distributions from batches.

The general formulation also allows us to extend robust distribution-estimation results to other learning tasks. 
We apply this framework to derive the first robust classification results,
where the goal is to minimize the excess risk in comparison to the best hypothesis, in the presence of adversarial batches. 
We obtain tight upper bounds on the excess risk and number of samples required to achieve it for general binary classification problems. 
We then apply the results to derive a computationally efficient algorithm for hypotheses consisting of $k$ one-dimensional intervals using only $\cO(k)$ samples. 

%\subsection{Organization}
The rest of the paper is organized as follows. 
Section~\ref{sec:results} describes the paper's main technical results and their applications to distribution estimation and classification. 
Section~\ref{sec:preliminaries} introduces basic notation and techniques.
Section~\ref{sec:Vc} recounts basic tools from VC theory used to derive the results.
Section~\ref{sec:algmajg} derives a framework for robust distribution estimation in $\cF$-distance from corrupt and adversarial sample batches, and obtains upper bounds on the estimation accuracy and sample complexity. 
Finally, section~\ref{sec:F_k}, develops computationally efficient algorithms for learning in $\cF_k$ distance.

\subsection{General related work}
The current results extend several long lines of work on estimating structured distributions, including~\cite{o2016nonparametric,diakonikolas2016learning,ashtiani2018some}.
%For non robust setting... this is the line of long work. First devory, then structured distribution in one dimension and high dimension.. survey in this.
The results also relate to classical robust-statistics work~\cite{tukey1960survey,huber1992robust}. There has also been significant recent work leading to practical distribution learning algorithms 
that are robust to adversarial contamination of the data.
For example,~\cite{diakonikolas2016robust,lai2016agnostic} presented algorithms for learning the mean and covariance matrix of high-dimensional sub-gaussian and other
distributions with bounded fourth moments 
in presence of the adversarial samples. Their estimation guarantees are typically in terms of $L_2$, and do not yield the $L_1$- distance results required for discrete distributions. 

The work was extended in~\cite{charikar2017learning} to the case when more than half of the samples are adversarial.
Their algorithm returns a small set of candidate distributions one of which is a good approximate of the underlying distribution. For more extensive survey on robust learning algorithms in the continuous setting,  see~\cite{steinhardt2017resilience,diakonikolas2019robust}.

Another motivation for this work derives 
from the practical federated-learning problem, where information arrives in batches~\cite{mcmahan2016communication,mcmahan2017mac}.

%\subsection{Problem formulation} 
%Estimating a distribution $p$ to an $\mathcal{F}$-distance $\epsilon$ means to find a distribution that is within $\epsilon$ $\cF$ distance from it. %that the estimated probability of any subset $S\in\cF$ is within $\epsilon$ from its actual probability. 
%\tcr{
%We first view this problem from an information theoretic perspective and show tight upper bounds on the $\cF$-distance achievable and the number of samples required to achieve this bounds, both bounds matches the lower bound to a small log factors. \\
%Then we focus on the following family of subsets of reals. Let $\cF _k$ be the collection of all unions of $k$ real intervals. Note that a subset of $\reals$ is in $\cF _k$ iff it is a union of at most $k$ disjoint non-empty intervals, and that $\cF _k\subseteq\cF _{k+1}$ for all $k$. \\
%For family $\cF=\cF_k$, we give computationally efficient algorithm to achieve the above information theoretic upper bound on the distance and requires only $(1/\epsilon)$ times the number of samples compared to the sample complexity upper bounds.}

\section{Results}
\label{sec:results}
%Let $\goodfrac=\bgoodsize/\btotal$, 
%and $\advfrac=\badvsize/\btotal=1-\goodfrac$ 
%be the fractions of good and adversarial batches, respectively. 

We consider learning from batches of samples, when a $\advfrac-$fraction of batches are adversarial.

More precisely, 
$B$ is a collection of $m$ batches, composed of two \emph{unknown} sub-collections. 
A \emph{good sub-collection} $\bgood\subseteq B$ of $\ge(1-\advfrac)m$ \emph{good batches}, 
where each batch $b$ consists of $n$ independent samples, all distributed according to the same distribution 
$\probestbatch$ satisfying $||\probestbatch- \targetdis||_{TV}\le \targetdistance $.
And an \emph{adversarial sub-collection} $\badv = \allbatches \setminus \bgood$
of the remaining $\le\advfrac m$ batches, each consisting of the same number 
$n$ of arbitrary $\Omega$ elements, that for simplicity we call \emph{samples} as well. 
Note that the adversarial samples may be chosen in any way,
including after observing the the good samples.

Section 2.1 of~\cite{jain2019robust} shows that for discrete domains, 
results for the special case $\targetdistance = 0$, where all batch distributions
$\probestbatch$ are the target distribution $\targetdis$, 
can be easily extended to the general $\targetdistance>0$ case.
The same can be shown for our more general result,
hence for simplicity we assume that $\targetdistance = 0$.

The next subsection describes our main technical results for learning in $\cF$ distance. The subsections thereafter derive applications of these results for learning distributions in total variation distance and for binary classification.  

\subsection{Estimating distributions in $\cF$ distance}
Let $\cF$ be a family of subsets of a domain $\Samplespace$.
The $\cF$-\emph{distance} between two distributions $p$ and $q$ over $\Samplespace$ is the largest difference between the probabilities
$p$ and $q$ assign to any subset in $\cF$,
\[
||\targetdis-q||_\cF \triangleq \sup_{S\in \cF}|\targetdis(S)-q(S)|.
\]
The $\cF$-distance clearly generalizes the total-variation 
and $L_1$ distances. 
For the collection $\Sigma$ of all subsets of $\Omega$, 
$
||p-q||_{\Sigmafield}
=
||p-q||_{\text{TV}}
=
\textstyle{\half} ||p-q||_1.
$

Our goal is to use samples generated by a target distribution $\targetdis$ to approximate it to a small $\cF$-distance. 
For general families $\cF$, this goal cannot be accomplished even with
just good batches. 
Let $\cF=\Sigma$ be the collection of all subsets of the real interval domain $\Omega=[0,1]$. 
For any total number $t$ of samples, with high probability, it is impossible to distinguish the uniform distribution over $[0,1]$ from a uniform discrete distribution over a random collection of $\gg t^2$ elements in $[0,1]$. Hence any
estimator must incur TV-distance 1 for some distribution. 

This difficulty is addressed by Vapnik-Chervonenkis (VC) Theory. 
The collection $\cF$ \emph{shatters} a subset $S\subseteq\Samplespace$
if every subset of $S$ is the intersection of $S$ with a subset in $\cF$.
The VC-dimension $V_\cF$ of $\cF$ is the size of the largest subset shattered by $\cF$.
\begin{comment}
Consider a family $\cF$ of subsets of a sample space, or domain, $\Samplespace$.
A set $S\subseteq\Samplespace$ is \emph{shattered} by $\cF$ if each of its
subsets is the intersection of $S$ with a subset in $\cF$.
The VC-dimension $V_\cF$ of $\cF$ is the size of the largest subset shattered by $\cF$.
\end{comment}

Let $X^t=X_1\upto X_t$, be independent samples from a distribution $p$.
The empirical probability of $S\subseteq\Omega$ is
\[
\bar\targetdis_{t}(S)
\ed
\frac{|\sets{i:X_i\in S}|}{t}.
\]

The fundamental \emph{Uniform deviation inequality} of VC theory~\cite{vapnik1971uniform,talagrand1994sharper}
states that if $\cF$ has finite VC-dimension $V_\cF$, then 
$\bar{\targetdis}_t$ estimates $\targetdis$ well in $\cF$ distance.
For all $\delta > 0$, with probability $1-\delta$,
\[
 ||\targetdis-\bar\targetdis_{t}||_\cF \le \cO\Paren{\sqrt{\frac{V_\cF+\log 1/\delta}{t}}}.
\]
It can also be shown that $\bar\targetdis_t$ achieves the lowest possible 
$\cF$-distance, that we call the \emph{information-theoretic limit}.
%\emph{statistical estimation limit}.

In the adversarial-batch scenario, a fraction $\advfrac$ of the batches may be corrupted. It is easy to see that for any number $m$ of batches, however large, the adversary can cause $\bar\targetdis_{t}$ to approximate $\targetdis$ to $\cF$-distance $\ge\advfrac$, namely 
$||\empprob{t}-\targetdis||_\cF\ge\advfrac$.

Let $\empprob{\Bsc}$ be the empirical distribution induced by the
samples in a collection $\Bsc\subseteq\allbatches$. 
Our first result states that if $\cF$ has a finite VC-dimension, 
for $\btotal=\tilde\cO(V_\cF/\advfrac^2)$ batches, 
$B$ can be "cleaned" to a sub-collection $B'$ where
$||\VUempprob-\targetdis||_\cF=\tilde\cO(\advfrac/\sqrt n)$,
recovering $\targetdis$ with a simple empirical estimator.

\begin{theorem}\label{th:main1}
For any $\cF$, $\bsize$, $\advfrac\le 0.4$, $\delta>0$, and 
$\btotal\ge \cO\Paren{\frac{V_\cF \log(\bsize/\advfrac) +\log 1/\delta}{\advfrac^2 }}$,
there is an algorithm that with probability $\ge\! 1\!-\!\delta$ returns a sub-collection $\Bsc\!\subseteq\!B$ such that $|\Bsc\!\cap\!\bgood| \ge (1 -\frac\advfrac6) \bgoodsize$ and 
\[
||\VUempprob-\targetdis||_\cF \le \cO\Paren{\advfrac \sqrt{\frac{\log (1/\advfrac)}{\bsize}}}.
\] 
\end{theorem}
The $\cF$-distance bound matches the adversarial limit up to a small $O(\sqrt{\log (1/\advfrac)})$ factor. The bound on the number $\btotal$ of batches required to achieve this bound is also tight up to a logarithmic factor. 

The theorem applies to all families with finite VC dimension, and 
like most other results of this generality, it is necessarily non-constructive in nature.
Yet it provides a road map for constructing efficient algorithms for many
specific natural problems. 
In Section~\ref{sec:F_k} we use this approach to derive a polynomial-time
algorithm that learns distributions with respect to 
one of the most important and practical VC classes,
where $\Omega=\reals$, and $\cF=\cF_k$ is the collection of all
unions of at most $k$ intervals. 

\begin{theorem}\label{th:main2}
For any $\bsize$, $\advfrac\le 0.4$, $\delta>0$, $k>0$, and
$\btotal\ge \cO\Paren{\frac{k \log(\bsize/\advfrac) +\log 1/\delta}{\advfrac^3 }\cdot\sqrt{\bsize}}$,
there is an algorithm that runs in time polynomial in all parameters,
%$k,\bsize,1/\advfrac, 1/\delta$ and $m$, 
and with probability $\ge 1- \delta$ returns a sub-collection $\Bsc\subseteq\allbatches$,
such that $|\Bsc\cap\bgood| \ge (1 -\frac\advfrac6) \bgoodsize$ and 
\[
||\VUempprob-\targetdis||_{\cF_k} \le \cO\Paren{\advfrac \sqrt{\frac{\log (1/\advfrac)}{\bsize}}}.
\] 
\end{theorem}

The sample complexity in both the theorems are independent of the domain and depends linearly on the VC dimension of the family $\cF$.

\subsection{Approximating distributions in total-variation distance}

Our ultimate objective is to estimate 
the target distribution in total variation (TV) distance, one of the most common measures in distribution estimation. In this and the next subsection, we follow a framework developed in~\cite{devroye2001combinatorial}, see also ~\cite{diakonikolas2016learning}. 

The sample complexity of estimating   distributions in TV-distance grows with the domain size, becoming infeasible for large discrete domains and impossible for continuous domains. A natural approach to address this intractability is to assume that the underlying distribution belongs to, or is near, a structured class $\cP$ of distributions. 

Let $\text{opt}_{\cP}(p)\triangleq \inf_{q\in\cP}||p-q||_{TV}$ be the TV-distance of $p$ from the closest distribution in $\cP$. For example, for $p\in \cP$, $\text{opt}_{\cP}(p)=0$.
Given $\epsilon,\delta> 0$, we try to use samples from $p$ to find an estimate $\hat p$ such that, with probability $\ge 1-\delta$,
\[
||p-\hat p||_{TV} \le \alpha\cdot\text{opt}_{\cP}(p)+\epsilon
\]
for a universal constant $\alpha\!\ge\!1$, namely, to approximate $p$ about as well as the closest distribution in $\cP$.

Following~\cite{devroye2001combinatorial}, we utilize a connection between distribution estimation and VC dimension. 
Let $\cP$ be a class of distributions over $\Samplespace$. 
The \emph{Yatracos class}~\cite{yatracos1985rates} of $\cP$ is the family of $\Samplespace$  subsets 
\[
\cY(\cP)\triangleq \sets{\sets{\omega\in\Samplespace:p(\omega)\ge q(\omega)}:\,{p,q\in\cP}}.
\]
It is easy to verify that for distributions $p,q\in \cP$,
\[
||p-q||_{TV} = ||p-q||_{\cY(\cP)}.    
\]

The \emph{Yatracos minimizer} of a  distribution $p$ is its closest 
distribution, by $\cY(\cP)$-distance, in $\cP$,
\[
\psi_{\cP}(p) = \arg\min_{q\in \cP}||q-p||_{\cY(\cP)},
\]
where ties are broken arbitrarily. 
Using this definition and equations, and a sequence of triangle inequalities, Theorem 6.3 in~\cite{devroye2001combinatorial} shows that, for any distributions $p$, $p'$, and any class $\cP$,
\begin{align}\label{eq:samze}
    ||p-\psi_\cP(p')||_{TV} \le 3 \cdot\text{opt}_{\cP}(p)+ 4 ||p-p'||_{\cY(\cP)}.
\end{align}
Therefore, given a distribution that approximates $p$ in $\cY(\cP)$-distance, it is possible to find a distribution in $\cP$  approximating $p$ in TV-distance. In particular, when $p\in\cP$, the $\text{opt}$ term is zero. 

If the Yatracos class $\cY(\cP)$ has finite 
VC dimension, the VC Uniform deviation inequality ensures that for the empirical distribution $p'$ of i.i.d. samples from $p$, $||p'-p||_{\cY(\cP)}$ decreases to zero, and can be used to approximate $p$ in TV-distance.
This general method has lead to many sample- and computationally-efficient algorithms for estimating structured distributions in TV-distance. 

However, as discussed earlier, with a $\advfrac$-fraction of adversarial batches, the empirical distribution of all samples can be at a 
${\cY(\cP)}$-distance as large as  $\Theta(\advfrac)$ from $p$, leading to a large TV-distance.

Yet Theorem~\ref{th:main1} shows that data can be "cleaned" to remove outlier batches and retain batches whose empirical distribution approximates $p$ to a much smaller ${\cY(\cP)}$-distance of $\cO(\advfrac\sqrt{(\log 1/\advfrac)\bsize})$.
Combined with Equation~\eqref{eq:samze}, we
obtain a much better approximation of $p$ in total variation distance.

\begin{theorem}\label{th:vc}
For a distribution class $\cP$ with Yatracos Class of finite VC dimension $v$, 
for any $\bsize$, $\advfrac\le 0.4$, $\delta>0$, and 
$\btotal\ge \cO\Paren{\frac{v \log(\bsize/\advfrac) +\log 1/\delta}{\advfrac^2 }}$,
there is an algorithm that with probability $\ge\! 1\!-\!\delta$ returns a distribution $p'\in\cP$ such that
\[
||\targetdis-p'||_{TV} \le 3 \cdot\text{opt}_{\cP}(p) + \cO\Paren{\advfrac \sqrt{\frac{\log (1/\advfrac)}{\bsize}}}.
\] 
\end{theorem}

The estimation error achieved in the theorem for TV-distance matches the lower to a small logarithmic factor of $O(\sqrt{\log (1/\advfrac)})$, and is valid for any class $\cP$ with finite VC Dimensional Yatracos Class.

%examples of finite Yatracos's class.... 
Moreover, the upper bound on the number of samples (or batches) required by the algorithm to estimate $p$ to the above distance matches a similar general upper bound obtained for non adversarial setting to a log factor. This results for the first time shows that it is possible to learn a wide variety of distributions robustly using batches, even over continuous domains.

The theorem describes the rate at which $p$ can be learned in TV-distance.
This rate mathces the similar upper bound for non-adversarial seeting to a small logarithmic factor of $O(\sqrt{\log (1/\advfrac)})$, and is valid for any class $\cP$ with finite VC Dimensional Yatracos Class.
%examples of finite Yatracos's class.... 
Moreover, the upper bound on the number of samples (or batches) required by the algorithm to estimate $p$ to the above distance matches a similar general upper bound obtained for non adversarial setting to a log factor. This results for the first time shows that it is possible to learn a wide variety of distributions robustly using batches, even over continuous domains.
%\tcr{Review}

\subsection{Learning univariate structured distributions}
We apply the general results in the last two subsections to 
estimate distributions over the real line. 
We start with one of the most studied, and important, distribution families, the class of piecewise-polynomial distributions, and then observe that it can be generalized to even broader classes.

A distribution $p$ over $[a,b]$ is  $t$-piecewise, degree-$d$, if there is a partition of $[a,b]$ into $t$ intervals $I_1\upto I_t$, and degree-$d$ polynomials $r_1\upto r_t$ such that $\forall j$ and $x\in I_j$, $p(x) = r_j(x)$. The definition extends naturally to discrete distributions over $[k]=\sets{1\upto k}$.

Let $\cP_{t,d}$ denote the collection of all $t$-piece-wise degree $d$ distributions. 
$\cP_{t,d}$ is interesting in its own right, as it contains important distribution classes such as histograms. In addition, it approximates other important distribution classes, such as monotone,  log-concave, Gaussians, and their mixures, arbitrarily well, \eg~\cite{acharya2017sample}.
  
Note that for any two distributions $p,q\in \cP_{t,d}$,
the difference $p-q$ is a $2t$-piecewise degree-$d$ polynomial, hence every set in the Yatracos class of $\cP_{t,d}$,
\[
\sets{x\in\reals:p(x)\ge q(x)} = \sets{x\in\reals:p(x)-q(x)\ge 0}
\]
is the union of at most $2t\cdot d$ intervals in $\reals$. 
Therefore, $\cY(\cP_{t,d})\subseteq\cF_{2t\cdot d}$. And since $V_{\cF_k} = O(k)$ for any $k$, 
$\cY(\cP_{t,d})$ has VC dimension $\cO(td)$.

Theorem~\ref{th:vc} can then be applied to show that any target distribution $\targetdis$ can be estimated by a distribution in $\cP_{t,d}$ to a TV-distance that is within a small $\sqrt{\log(1/\advfrac)}$ factor from adversarial lower bound, using a number of samples, and hence batches, 
that is within a logarithmic factor from the information-theoretic lower bound~\cite{chan2014efficient}.

\begin{corollary}
Let $p$ be distribution over $\reals$.
For any $\bsize$, $\advfrac\le 0.4$, $t$, $d$, $\delta>0$, and 
$\btotal\ge \cO\Paren{\frac{td\log(\bsize/\advfrac) +\log 1/\delta}{\advfrac^2 }}$,
there is an algorithm that with probability $\ge\! 1\!-\!\delta$ returns a distribution $p'\in\cP_{t,d}$ such that
\[
||\targetdis-p'||_{TV} \le 3 \cdot\text{opt}_{\cP_{t,d}}(p) + \cO\Paren{\advfrac \sqrt{\frac{\log (1/\advfrac)}{\bsize}}}.
\] 
\end{corollary}

Next we provide a polynomial-time algorithm for estimating $p$ to  the same $\cO(\advfrac\sqrt{(\log 1/\advfrac)/\bsize})$
TV-distance, but with an extra $\cO(\sqrt{n}/\advfrac)$ factor in sample complexity.

Theorem~\ref{th:main2} provides a polynomial time algorithm that returns a sub-collection $\Bsc\subseteq\allbatches$ of batches whose empirical distribution $\VUempprob$ is close to $\targetdis$ in $\cF_{2td}$-distance.  \cite{acharya2017sample} provides a polynomial time algorithm that for any distribution $q$ returns a distribution in $\hat p \in \cP_{t,d}$ minimizing $||\hat p - q||_{\cF_{2td}}$ to an additive error. Then 
%choosing $q=\VUempprob$,
Equation~\eqref{eq:samze} and Theorem~\ref{th:main2} yield the following result.  

\begin{theorem} Let $p$ be any distribution over $\reals$. 
For any $\bsize$, $\advfrac\le 0.4$, $t$, $d$, $\delta>0$, and 
$\btotal\ge \cO\Paren{\frac{td\log(\bsize/\advfrac) +\log 1/\delta}{\advfrac^3 }\cdot\sqrt{\bsize}}$,
there is a polynomial time algorithm that with probability $\ge\! 1\!-\!\delta$ returns a distribution $p'\in\cP_{t,d}$ such that
\[
||\targetdis-p'||_{TV} \le \cO( \text{opt}_{\cP_{t,s}}(p))
%?
+ \cO\Paren{\advfrac \sqrt{\frac{\log (1/\advfrac)}{\bsize}}}.
\] 
\end{theorem}

\subsection{Binary classification}
The framework developed in this paper extends beyond distribution estimation. 
Here we describe its application to Binary classification.
Consider a family $\mathcal{H} : \Samplespace\to\sets{0,1}$ of Boolean functions, and a distribution $p$ over $\Samplespace\times\sets{0,1}$. Let $(X,Y)\sim\targetdis$, where $X\in\Samplespace$ and $Y\in\sets{0,1}$. The loss of hypothesis $h\in\mathcal{H}$ for distribution $p$ is
\[
r_{p}(h) = \textstyle{\Pr_{(X,Y)\sim p}}[h(X)\neq Y].
\]
The \emph{optimal classifier} for distribution $p$ is
\[
h^*(p)= \arg\min_{h\in\mathcal{H}}r_{p}(h),
\]
and the \emph{optimal loss} is
\[
r^*_{p}(\mathcal{H}) = r_{p}(h^*(p)).
\]

The goal is to return a hypothesis $h\in\mathcal H$ whose loss $r_{p}(h)$ is close to the optimal loss $r^*_{p}(\mathcal{H})$. 
%Let $\Samplespace'=\Samplespace\times\sets{0,1}$. 

Consider the following natural extension of VC-dimension from families of subsets to families of Boolean functions.
For a boolean-function family $\mathcal{H}$, define the family 
\[
\cF_{\mathcal{H}}\triangleq \sets{(\sets{\omega\in\Samplespace: h(\omega) = z},y): h\in \mathcal{H}, y,z\in\sets{0,1}}
\]
of subsets of $\Samplespace\times\sets{0,1}$, and let the VC dimesnsion of $\mathcal{H}$ be $V_{\mathcal{H}}\triangleq V_{\cF_{\mathcal{H}}}$.

The largest difference between the loss of a classifier for
two distributions $p$ and $q$ over $\omega\times\sets{0,1}$ is related to their $\cF_{\mathcal{H}}$-distance,
\begin{align}
\sup_{h\in\mathcal{H}} |r_p(h)- r_q(h)| &= \sup_{h\in\mathcal{H}}|\textstyle{\Pr_{(X,Y)\sim p}}[h(X)\neq Y] -  \textstyle{\Pr_{(X,Y)\sim q}}[h(X)\neq Y]|\nonumber\\
&\le  \sup_{h\in\mathcal{H}}\sum_{y\in\sets{0,1}}|\textstyle{\Pr_{(X,Y)\sim p}( h(X)=\bar y, Y=y )}-  \textstyle{\Pr_{(X,Y)\sim q}( h(X)=\bar y, Y=y )}|\nonumber\\
&\le  2 ||p-q||_{\cF_{\mathcal{H}}}.\label{eq:apw}
\end{align}

The next simple lemma, proved in the appendix, upper bounds the excess loss of the optimal classifier in $\mathcal{H}$ for a distribution $q$ for another distribution $p$ in terms of $\cF_{\mathcal{H}}$ distance between the distributions.
\begin{lemma}\label{lem:relloss}
For any two distributions $p$ and $q$ and hypothesis class $\mathcal{H}$,
\[
r_p(h^*(q))-r^*_p(\mathcal{H}) \le 4 ||p-q||_{\cF_{\mathcal{H}}}.
\]
\end{lemma}
%The next inequality shows that if a distribution $q$ approximates another distribution $p$ in $\cF_{\mathcal{H}}$ distance, then an optimal classifier for $q$ is also a near optimal classifier for $p$. 

When $q$ is the empirical distribution of non-adversarial i.i.d. samples
from $p$, $h^*(q)$ is called the \emph{empirical risk minimizer}, and the excess loss of the empirical risk minimizer in the above equation goes to zero if VC dimension of $\mathcal{H}$ is finite.

Yet as discussed earlier, when a $\advfrac$-fractions of the batches, and hence samples, are chosen by an adversary, the empirical distribution of all samples can be at a large $\cF_{\mathcal{H}}$-distance $\cO(\advfrac)$ from $p$, leading to an excess classification loss up to $\cO(\advfrac)$ for 
the empirical-risk minimizer. 

Theorem~\ref{th:main1} states that the collection of batches can be  
"cleaned" to obtain a sub-collection whose empirical distribution has a lower $\cF_\mathcal{H}$-distance from $\targetdis$.
The above lemma then implies that the optimal classifier for the empirical distribution of the cleaner batches will have a small excess risk for $p$ as well. The resulting non-constructive algorithm has excess risk and sample complexity that are optimal to a logarithmic factor.
\begin{theorem}\label{th:mainclas}
For any $\mathcal{H}$, $\bsize$, $\advfrac\le 0.4$, $\delta>0$, and 
$\btotal\ge \cO\Paren{\frac{V_{\mathcal{H}} \log(\bsize/\advfrac) +\log 1/\delta}{\advfrac^2 }}$,
there is an algorithm that with probability $\ge\! 1\!-\!\delta$ returns a sub-collection $\Bsc\!\subseteq\!B$ such that $|\Bsc\!\cap\!\bgood| \ge (1 -\frac\advfrac6) \bgoodsize$ and 
\[
r_p(h^*(\VUempprob))-r_p^*(\mathcal{H}) \le \cO\Paren{\advfrac \sqrt{\frac{\log (1/\advfrac)}{\bsize}}}.
\] 
\end{theorem}

To derive a computationally efficient algorithm, we focus on the following class of binary functions. 
For $k\ge 0$ let $\mathcal{H}_k$ denote the collection of all binary functions over $\reals$ whose decision region, namely values mapping to 1,
consists of at most $k$-intervals.
The VC dimension of $\cF_{\mathcal{H}_k}$ is clearly $\cO(k)$.

Theorem~\ref{th:main2} describes a polynomial time algorithm that returns a cleaner data w.r.t. $\cF_{\mathcal{H}_k}$ distance. From Lemma~\ref{lem:relloss}, the hypothesis that minimizes the loss for the empirical distribution of this cleaner data will have a small excess loss.
Furthermore,~\cite{maass1994efficient} derived a polynomial time algorithm to find the hypothesis $h\in \mathcal{H}_k$ that minimizes the loss for a given empirical distribution.
Combining these results, we obtain a computationally efficient classifier
in $\mathcal{H}_k$ that achieves the excess loss in the above theorem.

\begin{theorem}\label{th:mainclaspol}
For any $\mathcal{H}=\mathcal{H}_k$, $\bsize$, $\advfrac\le 0.4$, $\delta>0$, and 
$\btotal\ge \cO\Paren{\frac{k \log(\bsize/\advfrac) +\log 1/\delta}{\advfrac^3 }\cdot\sqrt{n}}$,
there is a polynomial time algorithm that with probability $\ge\! 1\!-\!\delta$ returns a sub-collection $\Bsc\!\subseteq\!B$ such that $|\Bsc\!\cap\!\bgood| \ge (1 -\frac\advfrac6) \bgoodsize$ and 
\[
r_p(h^*(\VUempprob))-r_p^*(\mathcal{H}_k) \le \cO\Paren{\advfrac \sqrt{\frac{\log (1/\advfrac)}{\bsize}}}.
\] 
\end{theorem}

\section{Preliminaries}
\label{sec:preliminaries}
We introduce terminology that helps describe the approach and results.
Some of the work builds on results in~\cite{jain2019robust}, and 
we keep the notation consistent. 

Recall that $\allbatches$, $\bgood$, and $\badv$ are the collections of 
all-, good-, and adversarial-batches. 
Let $\Bsc\subseteq\allbatches$, $\Gsc\subseteq\bgood$, and 
$\Asc\subseteq\badv$, denote sub-collections of all-, good-, and bad-batches.
We also let $\alphsubset$ denote a subset of the Borel $\sigma$-field $\Sigmafield$.

Let $X^{b}_1,X^{b}_2,...,X^{b}_n$ denote the $\bsize$ samples in a batch $b$, and let $\mathbf{1}_S$ denote the indicator random variable for a subset $S\in\Sigmafield$.
Every batch $b\in B$ induces an empirical measure $\Vbempprob$ over the domain $\Samplespace$, where for each $\alphsubset\in\Sigmafield$,
\[
\bempprob \triangleq \frac1\bsize\sum_{i\in [n]}\mathbf{1}_S(X_{i}^b).
\]
Similarly, any sub-collection $\Bsc\subseteq\allbatches$
of batches induces an empirical measure $\VUempprob$
defined by
\[
\Uempprob \triangleq  \frac1{|\Bsc|\bsize}\sum_{b\in \Bsc}\sum_{i\in [n]}\mathbf{1}_S(X_{i}^b) = \frac{1}{|\Bsc |}\sum_{b\in \Bsc }\bempprob.
\]
We use two different symbols to denote empirical distribution defined by single batch and a sub-collection of batches to make them easily distinguishable. 
Note that $\VUempprob$ is the mean of the empirical measures $\Vbempprob$ defined by the batches $b\in\Bsc$.

%For subset $\alphsubset\in\Sigmafield$, let $\med$ be the median of the set of estimates $\{\bempprob : b\in \allbatches \}$. Note that the median has been computed using the estimates $\bempprob$ for all the batches in $b\in \allbatches$.

Recall that $\bsize$ is the batch size.
For $r\in[0,1]$, let $\var r \triangleq  \frac{r(1-r)}{\bsize}$,  the variance of a Binomial$(r,n)$ random variable. 
Observe that 
\begin{equation}\label{eq:fineq}
\forall\, r,s\in [0,1],\, \var r\le \frac{1}{4\bsize}
\quad\text{ and }\quad
|\var r-\var s|\le \frac{|r-s|}\bsize,
\end{equation}
where the second property follows as 
$|r(1-r)-s(1-s)|=|r-s|\cdot|1-(r+s)|\le|r-s|$.

For $b\in \bgood$, the random 
variables $\mathbf{1}_S(X_{i}^b)$ for $i\in[\bsize]$
are distributed i.i.d. $\text{Bernoulli}(\subsetprobtarget)$, and 
since $\bempprob$ is their average, 
\[
E[\,\bempprob\,] = \subsetprobtarget \quad\text{ and } \quad \text{Var}[\,\bempprob\,]=
E[(\bempprob-\subsetprobtarget)^2] = \var\subsetprobtarget.
\]

For batch collection $\Bsc  \subseteq \allbatches $ and subset $\alphsubset\in\Sigmafield$, the empirical probability $\bempprob$ of $\alphsubset$ will vary with the batch $b\in \Bsc$.
The \emph{empirical variance} of these empirical probabilities is
\[
 \empvarsub {\Bsc} \triangleq  \frac{1}{|{\Bsc}|}\sum_{b\in {\Bsc}}(\bempprob- \Uempprob)^2. 
\]

\section{Vapnik-Chervonenkis (VC) theory}~\label{sec:Vc}
We recall some basic concepts and results in VC theory, and derive some of their simple consequences that we use later in deriving our main results.

The \emph{VC shatter coefficient} of $\cF$ is 
\[
S_{\cF }(t)
\ed
\sup_{x_1,x_2,..,x_t\in\Samplespace}|\{\{x_1,x_2,..,x_t\}\cap S : S\in\cF \}|,
\]
the largest number of subsets of $t$ elements in $\Samplespace$ obtained by intersections with subsets in $\cF $. 
The VC dimension of $\cF $ is 
\[
V_{\cF }
\ed
\sup\{t: S_{\cF }(t)= 2^t\},
\] 
the largest number of $\Samplespace$ elements that are "fully shattered" by $\cF $.
The following Lemma~\cite{devroye2001combinatorial} bounds the Shatter coefficient for a VC family of subsets.
\begin{lemma}[\cite{devroye2001combinatorial}]
For all $t \ge V_{\cF }$,\quad
$S_{\cF }(t)
\le\Paren{\frac{t\,e}{V_{\cF }}}^{V_{\cF }}$. 
\end{lemma}

Next we state the VC-inequality for relative deviation~\cite{vapnik1974theory,anthony1993result}.
\begin{theorem}\label{th:vcrel}
Let $\targetdis$ be a distribution over $(\Samplespace,\Sigmafield)$, and $\cF$ be a VC-family of subsets of $\Omega$ and $\bar\targetdis_t$ denote the empirical distribution from $t$ i.i.d samples from $\targetdis$. Then for any $\epsilon > 0$, with probability $\ge 1-8S_{\cF }(2t)e^{-t\epsilon^2/4}$,
\begin{align*}
\sup_{S\in\cF } \max \Big\{\frac{\bar\targetdis_t (S) - \subsetprobtarget}{\sqrt{\bar\targetdis_t (S) }},\frac{\subsetprobtarget-\bar\targetdis_t (S) }{\sqrt{\subsetprobtarget }}\Big\}\le {\epsilon}.
\end{align*}

\ignore{
\begin{align*}
\sup_{S\in\cF}\frac{\bar\targetdis_t (S) - \subsetprobtarget}{\sqrt{\bar\targetdis_t (S) }}< {\epsilon}\ \  \text{ and }\ \  \sup_{S\in\cF }\frac{\subsetprobtarget-\bar\targetdis_t (S) }{\sqrt{\subsetprobtarget }}\le {\epsilon}.
\end{align*}
}

\end{theorem}

Another important ingredient commonly used in VC Theory is the concept of covering number that reflects the smallest number of subsets that approximate each subset in the collection. 

Let $\targetdis$ be any probability measure over $(\Samplespace,\Sigmafield)$ 
and $\cF \subseteq\Sigmafield$ be a family of subsets.
A collection of subsets {$\cC\subseteq\Sigmafield$} is an \emph{$\epsilon$-cover} of $\cF $ if for any $S\in 
\cF $, there exists a $S ' \in\cC$ with $\targetdis(S\triangle S ') \le \epsilon$.
The \emph{$\epsilon$-covering number} of $\cF $ is
\[
N(\cF,\targetdis,\epsilon)
\triangleq
\inf\{|\cC|: \cC\text{ is an }\epsilon\text{-cover of } \cF \}.
\]

%Any epsilon cover $\cC$ of $\cF $ such that $\cC\subseteq\cF $ is also referred as \emph{$\epsilon$-self cover} of $\cF $.

If $\cC\subseteq \cF $ is an \emph{$\epsilon$-cover} of $\cF $, then
$\cC$ is \emph{$\epsilon$-self cover} of $\cF $.

The \emph{$\epsilon$-self-covering number} is
\[
N^s(\cF,\targetdis,\epsilon)
\triangleq
\inf\{|\cC|: \cC \text{ is an }\epsilon\text{-self-cover of } \cF \}.
\]
Clearly, $N^s(\cF,\targetdis,\epsilon)\ge N(\cF,\targetdis,\epsilon)$.
The next lemma establishes a reverse relation.
\begin{lemma}\label{lem:covnumrel}
For any $\epsilon\ge0$, $N^s(\cF,\targetdis,\epsilon)\le N(\cF,\targetdis,\epsilon/2)$.
\end{lemma}
\begin{proof}
If $N(\cF,\targetdis,\epsilon/2)=\infty$, the lemma clearly holds.
Otherwise, let $\cC$ be an $\epsilon/2$-cover of size $N(\cF,\targetdis,\epsilon/2)$. We construct an 
$\epsilon$-self-cover of equal or smaller size. 

For every subset $S_\cC\in\cC$, there is a subset $S=f(S_\cC)\in\cF $ with 
$\targetdis(S_\cC\,\triangle\, f(S_\cC))\le\epsilon/2$.
Otherwise, $S_\cC$ could be removed from $\cC$ to obtain a strictly smaller $\epsilon/2$ cover, which is impossible. 

The collection $\sets{f(S_\cC):S_\cC\in\cC}\subseteq\cF $ 
has size $\le|\cC|$, and it is an $\epsilon$-self-cover of $\cF $ because for any $S\in\cF $, there is an $S_\cC\in\cC$ with 
$\targetdis(S\,\triangle\, S_\cC)\le\epsilon/2$, and
by the triangle inequality, 
$\targetdis\big(S\,\triangle\, f(S_\cC)\big)\le\epsilon$.
\end{proof}

Let $N_{\cF,\epsilon } \triangleq \sup_{\targetdis} N(\cF,\targetdis,\epsilon)$ and $N^s_{\cF,\epsilon } \triangleq \sup_{\targetdis} N^s(\cF,\targetdis,\epsilon )$ be the largest covering numbers under any distribution. 

The next theorem bounds the covering number of $\cF $ in terms of its VC-dimension. 
\begin{theorem}[\cite{vaart1996weak}]\label{th:covnumbound}
There exists a universal constant $c$ such that for any $\epsilon>0$,
and any family $\cF $ with VC dimension $V_\cF $, 
\[
N_{\cF,\epsilon }\le c V_\cF  \Big(\frac{4e}{\epsilon}\Big)^{V_\cF }.
\]
\end{theorem}
Combining the theorem and Lemma~\ref{lem:covnumrel}, we obtain the following corollary.
\begin{corollary}\label{cor:intcovbou}
\[
N^s_{\cF,\epsilon }\le c  V_\cF  \Big(\frac{8e}{\epsilon}\Big)^{ V_\cF }.
\]
\end{corollary}
For any distribution $\targetdis$ and family $\cF $, let $\mathcal{C}^s(\cF,\targetdis,\epsilon )$ be any 
minimal-size $\epsilon$-self-cover for $\cF $ of size $\le N^s_{\cF,\epsilon }$.

\section{A framework for distribution estimation from corrupted sample
batches}\label{sec:algmajg}
We develop a general framework to learn $\targetdis$ in $\cF$ distance and derive Theorem~\ref{th:main1}.
Recall that the $\cF $ distance between two distributions $p$ and $q$ is 
\[
||p-q||_{\cF }
=
\sup_{\alphsubset\in \cF } |p(\alphsubset)-q(\alphsubset)|.  
\]

The algorithms presented enhance the algorithm of~\cite{jain2019robust}, developed for $\cF=2^{\Samplespace}$ of a discrete domain $\Samplespace=[k]$, to any VC-family $\cF$ of subsets of any sample space $\Samplespace$. We retain the part of the analysis and notation that are common in our enhanced algorithm and the one presented in~\cite{jain2019robust}.

At a high level, we remove the adversarial, or "outlier" batches, and return a sub-collection $\Bsc\subseteq\allbatches$ of batches whose empirical distribution $\VUempprob$ is close to $p$ in $\cF$ distance. 
The uniform deviation inequality in VC theory states
that the sub-collection $\bgood$ of good batches has
empirical distribution $\bar{\targetdis}_{\bgood}$ 
that approximates $p$ in $\cF$ distance, 
thereby ensuring the existence of such a sub-collection.

The family $\cF$ can be potentially uncountable, hence
learning a distribution to a given $\cF$ distance may entail
simultaneously satisfying infinitely many constrains. 
To decrease the constraints to a finite number, 
Corollary~\ref{cor:intcovbou} shows that for any 
distribution and any $\epsilon> 0$, there exists a finite
$\epsilon$-cover of $\cF$ w.r.t this distribution. 

%Our goal is to find an $\epsilon$-cover of $\cF$ w.r.t. the target distribution $\targetdis$. Yet this distribution is unknown to us, and
%its samples provided in the collection $B$ of batches are corrupted by an adversary. 
%But the question is what distribution can be used to construct the $\epsilon$-cover $\cC$ of $\cF$ so that a small $\cC$-distance $||\VUempprob-\targetdis||_\cC$ would imply a small $\cF$-distance $||\VUempprob-\targetdis||_\cF$.

Our goal therefore is to find an $\epsilon$-cover $\cC$ of $\cF$ w.r.t.\ an appropriate distribution such that if for some sub-collection $\Bsc$ the empirical distribution $\VUempprob$ approximates $p$ in $\cC$-distance it would also approximate $p$ in $\cF$-distance. The two natural distribution choices are the target distribution $p$ or empirical distribution from its samples. Yet the distribution $p$ is unknown to us, and
its samples provided in the collection $B$ of batches are corrupted by an adversary.

The next theorem overcomes this challenge by 
showing that although the collection $\allbatches$ includes adversarial batches, for small enough $\epsilon$, for any $\epsilon$-cover $\cC$ of $\cF$ w.r.t. the empirical distribution $\empprob{\allbatches}$,  
a small $\cC$-distance $||\VUempprob-\targetdis||_\cC$, between $\targetdis$ and the empirical distribution induced by a sub-collection $\Bsc\subseteq\allbatches$ would imply a small $\cF$-distance $||\VUempprob-\targetdis||_\cF$, between the two distributions.

%The next theorem overcomes this challenge by 
%showing that although the collection $\allbatches$ includes adversarial batches, for small enough $\epsilon$, any $\epsilon$-cover $\cC$ of $\cF$ w.r.t. the empirical distribution $\empprob{\allbatches}$ approximates $\cF$ w.r.t. the target distribution $\targetdis$.
%In particular, the theorem shows that 
%a small $\cC$-distance $||\VUempprob-\targetdis||_\cC$, between $\targetdis$ and the empirical distribution induced by a sub-collection $\Bsc\subseteq\allbatches$ would imply a small $\cF$-distance $||\VUempprob-\targetdis||_\cF$, between the two distributions.

%Note that the theorem allows for the possibility that the subsets in the $\epsilon$-cover $\cC$ of $\cF$ can be from a family of subsets $\cF'$, larger than $\cF$.

Note that the theorem allows the $\epsilon$-cover $\cC$ of $\cF$ 
to include sets in the subset family $\cF'$ containing $\cF$.

\begin{theorem}~\label{th:relate}
For $\btotal\ge \cO(\frac{V_{\cF'}\log(\bsize/\advfrac)+\log(1/\delta)}{\advfrac^2})$ and $\epsilon \le \frac{\advfrac}{\sqrt{\bsize}}$, let $\mathcal{C}\subseteq \cF'$ be an $\epsilon$-cover of family $\cF $ w.r.t. the empirical distribution $\empprob{\allbatches}$. Then with probability $\ge 1-\delta$, for any sub-collection of batches  $\Bsc\subseteq B$ of size $|\Bsc|\ge \btotal/2$,
\[
||\bar\targetdis_{\Bsc}-\targetdis||_\cF \le ||\bar\targetdis_{\Bsc}-\targetdis||_\cC+\frac{5\advfrac}{\sqrt{\bsize}}.
\]
\end{theorem}
\begin{proof}
Consider any batch sub-collection $\Bsc\subseteq \allbatches$.
%Let $\alphsubset,\, \alphsubset'$ be such a pair of subsets.
For every $S,S'\subseteq\Samplespace$, by the triangle inequality,
\begin{align}
    |\VUempprob(\alphsubset)-\subsetprobtarget|& =
    \left|\Big(\VUempprob(\alphsubset^\prime)+\VUempprob(S\setminus \alphsubset^\prime)-\VUempprob(\alphsubset^\prime\setminus S)\Big)-\Big(\targetdis(\alphsubset^\prime)+ \targetdis(S\setminus \alphsubset^\prime)-\targetdis(\alphsubset^\prime\setminus S)\Big)\right|\nonumber\\
    & \le |\VUempprob(\alphsubset^\prime)-\targetdis(\alphsubset^\prime)|+\VUempprob(S\setminus \alphsubset^\prime)+\VUempprob(\alphsubset^\prime\setminus S)+ \targetdis(S\setminus \alphsubset^\prime)+\targetdis(\alphsubset^\prime\setminus S)\nonumber\\
    & = |\VUempprob(\alphsubset^\prime)-\targetdis(\alphsubset^\prime)|+\VUempprob(S\triangle \alphsubset^\prime)+\targetdis(S\triangle \alphsubset^\prime)\label{eq:e1}.
\end{align}
Since $\mathcal{C}$ is an $\epsilon$-cover  w.r.t. $\empprob{\allbatches}$, for every $\alphsubsetdef$ there is an $ \alphsubset^\prime\in \mathcal{C}$ such that $ \empprob{\allbatches}(S\triangle \alphsubset^\prime)\le \epsilon$.
For such pairs, we bound the second term on the right in the above equation. 
\begin{align}
\VUempprob(S\triangle \alphsubset^\prime)
&=\frac1{|\Bsc|\bsize}\sum_{b\in \Bsc}\sum_{i\in [n]}\mathbf{1}_{S\triangle \alphsubset^\prime}(X_{i}^b)\nonumber\\
&\le\frac1{|\Bsc|\bsize}\sum_{b\in B}\sum_{i\in [n]}\mathbf{1}_{S\triangle \alphsubset^\prime}(X_{i}^b)\nonumber\\
&= \frac{|B|}{|\Bsc|}\cdot\frac1{|B|\bsize}\sum_{b\in B}\sum_{i\in [n]}\mathbf{1}_{S\triangle \alphsubset^\prime}(X_{i}^b)\nonumber\\
&= \frac{m}{|\Bsc|} {\empprob{\allbatches}(S\triangle \alphsubset^\prime)}\le \frac{\btotal\epsilon}{|\Bsc|}.\label{eq:e2}
\end{align}
Choosing $\Bsc=\bgood$ in the above equation and using $\bgood = (1-\advfrac)\btotal\ge \btotal/2$ gives,
\begin{align}
   \empprob{\bgood}(S\triangle \alphsubset^\prime) < 2\epsilon .\label{eq:e3}
\end{align}
Then 
\begin{align*}
\targetdis(S\triangle \alphsubset^\prime) &\le| \targetdis(S\triangle \alphsubset^\prime)-\empprob{\bgood}(S\triangle \alphsubset^\prime)|+\empprob{\bgood}(S\triangle \alphsubset^\prime)\\
&\overset{\text{(a)}}\le \sup_{S,\,S'\in \cF'} | \targetdis(S\triangle \alphsubset^\prime)-\empprob{\bgood}(S\triangle \alphsubset^\prime)|+ 2\epsilon\\
&\overset{\text{(b)}}\le 2\epsilon+\frac{\advfrac}{\sqrt{\bsize}},
\end{align*}
with probability $\ge 1-\delta$, here (a) used equation~\eqref{eq:e3} and (b) follows from Lemma~\ref{lem:impofrd}.
Combining equations~\eqref{eq:e1},~\eqref{eq:e2} and the above equation completes the proof. 
\end{proof}

The above theorem reduces the problem of estimating in $\cF$ distance to finding a sub-collection $\Bsc\subseteq\allbatches$ of at least $\btotal/2$ batches such that for an $\epsilon$-cover $\cC$ of $\cF$ w.r.t. distribution $\empprob{\allbatches}$, the distance $||\VUempprob-\targetdis||_\cC$ is small. If we choose a finite $\epsilon$-cover $\cC$ of $\cF$, the theorem would ensure that the number of constrains is finite.

To find a sub-collection of batches as suggested above, we show that
with high probability, certain concentration properties hold for all
subsets in $\cF'$. 
Note that the cover $\cC$ is chosen after seeing the samples in
$\allbatches$, but since $\cC\subseteq\cF'$, the
results also hold for all subsets in $\cC$.

The following discussion develops some notation and intuitions that leads to these properties.  

We start with the following observation.
Consider a subset $\alphsubset\in \cF'$. For evey good batch $b\in \bgood$, $\bempprob$ has a sub-gaussian distribution  $\text{subG}(\subsetprobtarget,\frac{1}{4{\bsize}})$ with variance $\var{\subsetprobtarget}$. Therefore, most of the good batches $b\in\bgood$ assign the empirical probability  $\bempprob \in \subsetprobtarget \pm \tilde O(1/\sqrt{\bsize})$. Moreover, the empirical mean and variance of $\bempprob$ over $b\in\bgood$ converges to the expected values $\subsetprobtarget$ and $\var{\subsetprobtarget}$, respectively. 

In addition to the good batches, the collection $\allbatches$ of batches 
also includes an adversarial sub-collection $\badv$ of batches that constitute up to a $\advfrac-$fraction of $\allbatches$.
If the difference between $\subsetprobtarget$ 
and the average of $\bempprob$ over all adversarial batches $b\in\badv$ is $\le \tilde O(\frac{1}{\sqrt{\bsize}})$, 
namely comparable to the standard deviation of $\bempprob$ for the good batches $b\in\bgood$, then the adversarial batches can change the overall mean of empirical probabilities $\bempprob$ by at most $\tilde O(\frac{\advfrac}{\sqrt{\bsize}})$, which is within our tolerance.
Hence, the mean of $\bempprob$ will deviate significantly from $\subsetprobtarget$ only in the presence of a large number of adversarial batches $b\in\badv$ whose empirical probability $\bempprob$ differs from $\subsetprobtarget$ by $\gg\tilde O(\frac{1}{\sqrt{\bsize}})$.
%, displayed by the good batches. 

%We want to find such subsets $\alphsubset$ to identify the adversarial batches efficiently.

%In this section, we present a computationally efficient algorithm to recover the distribution $\targetdis$. 
%To separate the statistical part from the analysis of the algorithm, we assume that the sub-collection of good batches $\bgood$ satisfy certain deterministic conditions. 
%These conditions are shown to hold with a high probability for sub-collection of good batches in $\bgood$. Nothing is assumed about the adversarial batches, except that thy form a $\le\advfrac$ fraction of the overall batches.
%Our randomized algorithm~\ref{alg2} recovers the distribution $\targetdis$ with high probability.
%Like the simple estimate, the proposed algorithm~\ref{alg1} estimates the probability of each symbol by the empirical count, except that instead of considering all batches, the count is taken over an "appropriate" sub-collection $\Bsc_{f} \subseteq\allbatches$ of the batches.

%Our algorithm assumes a number of deterministic conditions holds simultaneously for a collection of good batches $\bgood$.
%To state these conditions, and the algorithm later, we state the following definitions.
To quantify this effect, for a subset $\alphsubset\in\cF'$ let
\[
\med \triangleq \text{median}\{\bempprob : b\in \allbatches\}
\]
be the median empirical probability of $\alphsubset$ over all batches. 
Property~\ref{con} shows that $\med$ is a good approximation of
$\subsetprobtarget$.
Define the \emph{corruption score}
%Because good batches are in majority in $\allbatches$, and for most of the good batches $|\bempprob - \subsetprobtarget| \le O(\sqrt{1/n})$, we get
%\[
%|\med-\subsetprobtarget|\le O(\sqrt{1/n}).
%\]
of batch $b$ for $\alphsubset$ to be
\[
\indcorruption
\triangleq
\begin{cases}
0 &\text{if }\ |\bempprob- \med| \le 3 \sqrt{\frac{\ln (6e/\advfrac)}\bsize} ,\\
(\bempprob- \med)^2& \text{otherwise}.
\end{cases}
\]
The preceding discussion shows that the corruption score of most good batches for a fixed subset $\alphsubset$ is zero, and that adversarial batches that may significantly change the overall mean of empirical probabilities have high corruption score.

The \emph{corruption score} of a sub-collection $\Bsc$ for a subset $\alphsubset$ is the sum of the \emph{corruption score} of its batches,
\[
\genoverallcorruptionsub  \triangleq \sum_{b\in \Bsc} \indcorruption.
\]
A high corruption score of $\Bsc$ for a subset $\alphsubset$ indicates that  $\Bsc$ has many batches $b$ with large difference $|\bempprob- \med|$.
Finally, the \emph{corruption score} of a sub-collection $\Bsc$ for a family of subsets $\cF''\subseteq\cF'$ is the largest corruption score of any $\alphsubset\in\cF''$, 
\[
{\corruption_{\Bsc}(\alphabet'')} \triangleq\max_{\alphsubset\in\cF''}\genoverallcorruptionsub .
\]

Note that removing batches from a sub-collection reduces its corruption. 
We can simply make corruption zero by removing all batches, but we would lose all the information as well. 
%In fact the goal of making corruption zero is not ideal as in the example we discussed at the start of this section, for each batch $|\bempprob-\subsetprobtarget|$ for at least some subsets, hence this goal is bound to loose all batches. 
As described later in this section, the algorithm reduces the corruption below a threshold by removing a few batches while not sacrificing too many good batches in the process.  
%The set $U(\alphsubset)$ is a collection of \emph{suspicious batches} based on the distance of empirical probability $\bempprob$ from the median of the estimates $\{\bempprob: b\in \allbatches\}$ for a set $\alphsubset$.

Recall that $\allbatches$ is a collection of $\btotal$ batches, each containing $\bsize$ samples, and that a sub-collection $\bgood \subseteq \allbatches$ consists
of $\ge (1-\advfrac)\btotal$ good batches where all samples are drawn from the target distribution $\targetdis$.
%, without any assumptions on the adversarial batches 
We show that regardless of the samples in adversarial batches, 
with high probability, $\allbatches$  satisfies the following three \emph{concentration properties}. 
\begin{enumerate}
\item\label{con} For all $\alphsubsetdef'$, the median of the estimates $\sets{\bempprob:b\in\allbatches}$
approximates $\subsetprobtarget$ well,
\[
|\med-\subsetprobtarget|\le \sqrt{\ln (6)/\bsize}.
\]
\item\label{con2} 
For every sub-collection $\Gsc\subseteq\bgood$ containing a large portion of the good batches, $|\Gsc| \ge (1-\advfrac/6)\bgoodsize$, and for all $\alphsubsetdef'$,
the empirical mean and variance of $\bempprob$ estimate $\subsetprobtarget$ and $\var{\subsetprobtarget}$ well,
\begin{align*}
&|\UGempprob - \subsetprobtarget | \le  \frac \advfrac 2\sqrt{\frac{\ln (6e/\advfrac) } {\bsize}},
\end{align*}
and
\begin{align*}
\Big|\frac1{|\Gsc |}  \sum_{b\in \Gsc } (\bempprob -   \subsetprobtarget)^2  - \var{\subsetprobtarget}\Big|\le {\frac{ 6\advfrac\ln (\frac{6e}\advfrac) } {\bsize}}.
  %\label{eq:goodvardiff}
\end{align*}
\item\label{con3} The corruption score of the collection $\bgood$ of good batches for family $\cF'$ is small,
    %\item\label{con3} The total corruption score of the good batches is small, namely 
\begin{align*}
{\corruption_{\bgood}(\alphabet')} \le {\frac{ \advfrac \btotal\ln ({6e}/\advfrac) } {\bsize}}\triangleq \kappa_G.
\end{align*}
\end{enumerate}

\begin{lemma}\label{lem:prophold}
Let $\cF'$ have finite VC dimension and 
$\btotal \ge  \cO(\frac{V_{\cF'}\log(\bsize/\advfrac)+\log(1/\delta)}{\advfrac^2})$.
With probability $\ge 1-\delta$ the three 
essential properties hold.
\end{lemma}
These properties extend the same properties in~\cite{jain2019robust} (Section 2) from subsets of discrete domains
to families of subset with finite VC dimension in any euclidean space.

To prove that properties hold with high probability, we first show that for an appropriately chosen epsilon, they hold for all subsets in a minimal-size $\epsilon$-cover of $\cF'$ w.r.t. the target distribution $\targetdis$.
Since the cover has finite size, a proof similar to the one in~\cite{jain2019robust}, for discrete domains, shows that the properties hold for all subsets in the cover. This uses the observation that for $b\in\bgood$, $\bempprob$ has a sub-gaussian distribution $\text{subG}(\subsetprobtarget,\frac{1}{4\bsize})$, and variance $\var{\subsetprobtarget}$. 
We then use Lemma~\ref{lem:impofrd} to extend the properties from subsets in the cover to all subsets in class $\cF'$ 
The proof is in Appendix~\ref{sec:goodprop}.

The remainder of this section assumes that the properties in the above lemma hold.

For any $\cC\in \cF'$ conditions~\ref{con},~\ref{con2} and~\ref{con3} holds for all subsets in $\cC$ w.h.p.

Next a simple adaptation of the Batch Deletion algorithm in~\cite{jain2019robust} is used to find a sub collection of batches $\Bsc$ such that $||\VUempprob - \targetdis ||_{\cC}$ is small.

For any $\cC\subseteq\cF'$, the next Lemma bounds $\cC$-distance of empirical distribution $\VUempprob$ in terms of the corruption of $\Bsc$ for sub-family $\cC$. 

\begin{lemma}\label{lem:corruptionandell1}
Suppose Properties~\ref{con}-~\ref{con3} hold. Then for any $\Bsc$ such that $|\Bsc\cap\bgood| \ge (1 -\frac\advfrac6) \bgoodsize$ and any family $\cC\subseteq\cF'$ such that $\corruption_{\Bsc}(\cC) \le t\cdot\kappa_G$, for some $t\ge 0$,
then
\[
||\VUempprob - \targetdis ||_{\cC}  \le  (5+1.5\sqrt{t}) \advfrac\sqrt{{\frac{\ln (6e/\advfrac)}\bsize} }.
\]
\end{lemma}
The proof of the above lemma is the same as the proof of a similar Lemma 4 in~\cite{jain2019robust}, hence we only give a high level idea here.
For any sub-collection $\Bsc$ retaining a major portion of good batches, from Property~\ref{con2}, the mean of $\Vbempprob$ of the good batches $\Bsc\cap\bgood$  approximates $\targetdis$.
Then showing that a small corruption score of $\Bsc$ w.r.t. all subsets $\alphsubset\subseteq\cC$ imply that the adversarial batches $\Bsc\cap\badv$ have limited effect on $\Uempprob$ proves the above lemma.

Next we describe
%a \emph{Batch Deletion} Algorithm
%\emph{Batch Deletion}~\cite{jain2019robust},
the \emph{Batch Deletion} Algorithm in~\cite{jain2019robust}.
Given a sub-collection $\Bsc$ and any subset $S\in \cF$, 
the algorithm successively removes batches from $\Bsc$,
invoking Property~\ref{con3} to ensure that each batch removed is adversarial with probability $\ge 0.95$.
The algorithm stops when the sub-collection's corruption score w.r.t. $S$ is at most $20\kappa_G$.
\begin{algorithm}[H]
%\begin{algorithm}
  \caption{Batch Deletion}\label{alg2}
  \begin{algorithmic}[1]
    \STATE {\bfseries Input:} Sub-Collection $\Bsc$ of Batches, subset $\alphsubset\subseteq\cF'$, $\medV$=$\med$, and $\kappa_G$
  \STATE {\bfseries Output:} A smaller sub-collection $\Bsc$ of batches
  \STATE {\bfseries Comment:} The terms $\kappa_G$, $\indcorruption$, and $\genoverallcorruptionsub$ used below are defined earlier in this section, and computing $\indcorruption$ and $\genoverallcorruptionsub$ require $\med$ as input.
   % \STATE $V = \{b\in \Bsc : |\bempprob- \medV| \ge 4 \sqrt{\frac{\ln (6e/\advfrac)}\bsize}) \} $; \\// \COMMENT{equivalent to $V\gets U(\alphsubset)\cap \Bsc$}
   % \STATE $\Sigma \gets \sum_{b\in V} (\bempprob- \medV)^2$.
    \WHILE{  $\genoverallcorruptionsub \, \geq\,  20\kappa_G$} 
        \STATE {Select a single batch $b\in \Bsc$ where batch $b$ is selected with probability $ \frac{\indcorruption}{\genoverallcorruptionsub }$};
        \STATE $\Bsc \gets \{\Bsc \setminus {b}\}$;
        %\STATE $\Sigma\ \gets \Sigma-(\bempprob- \medV)^2$;
    \ENDWHILE
    \STATE \textbf{return }$(\Bsc)$;
  \end{algorithmic}
\end{algorithm}

%(as it is a bit involved, we don't give the details here and refer the readers to the original paper.)

Given any finite $\cC\subseteq\cF'$, the next algorithm~\ref{alg3} uses Batch Deletion to successively update $\allbatches$ and decrease the corruption score for each subset $\alphsubsetdef$. 

%This Deletion can be called recursively for subsets with high corruption to update the collection of batches to obtain a 

Since each batch removed is adversarial with probability $\ge 0.95$ and the number of adversarial batches $\le\advfrac\btotal$, the the final sub-collection returned by the algorithm retains a large fraction of good batches. 

\begin{algorithm}[H]
%\begin{algorithm}
  \caption{}\label{alg3}
  \begin{algorithmic}[1]
    \STATE {\bfseries Input:} Collection $\allbatches$ of Batches, finite subset family $\cC\subseteq\cF'$, adversarial batches fraction $\advfrac$
  \STATE {\bfseries Output:} A sub-collection $\Bsc$ of batches
  \STATE {\bfseries Comment:} The terms $\kappa_G$, $\genoverallcorruptionsub$, and $\med$ used below are defined earlier in this section
  \STATE {$\Bsc = \allbatches$};
    \FOR{  $\alphsubset\in \cC$} 
        \IF{$\genoverallcorruptionsub \, \geq\,  25\kappa_G$ } 
        \STATE $\medV \gets \medV(\bar{\mu}(\alphsubset))$;
        \STATE $\Bsc\gets $\text{Batch Deletion}{($\Bsc,\alphsubset,\medV$)};
        \ENDIF
    \ENDFOR
    \STATE \textbf{return }$(\Bsc)$;
  \end{algorithmic}
\end{algorithm}

The next lemma characterizes the algorithm's performance. The proof of the lemma is immediate from the above discussion. 
\begin{lemma}
Suppose Properties~\ref{con},~\ref{con2} and~\ref{con3} hold. Let $\cC\subseteq\cF'$ be a finite family of subsets. Then algorithm~\ref{alg3} 
%calls Batch-Deletion only $O(n)$ times, and 
returns a sub-collection of batches $\Bsc$ such that with probability $\ge 1- e^{-O(\advfrac\btotal)}$,
$|\Bsc\cap\bgood| \ge (1 -\frac\advfrac6) \bgoodsize$ and $\corruption_{\Bsc}(\cC) \le 20\kappa_G$.
\end{lemma}

Next choose $\cF'=\cF$, and $\cC$ to be the $\epsilon$-self-cover of $\cF$. The above Lemma, Theorem~\ref{th:relate}, Lemma~\ref{lem:corruptionandell1}, and Lemma~\ref{lem:prophold} imply the following theorem that derives the upper bounds for robust distribution estimation from batches. 
%While the algorithm may not be efficient for all subset families, it provides insight on..
%In the next section we build on these technique and provide a polynomial-time algorithm for robust estimation of one of the most important and practical family of subsets. 

\begin{theorem}[Theorem~\ref{th:main1} restated]
For any given $\advfrac\le 0.4$, $\delta>0$, $\bsize$, $\cF$, and
$\btotal\ge \cO\Paren{\frac{V_\cF \log(\bsize/\advfrac) +\log 1/\delta}{\advfrac^2 }}$,
there is a non-constructive algorithm that with probability $\ge 1- \delta$ returns a sub-collection of batches $\Bsc$ such that $|\Bsc\cap\bgood| \ge (1 -\frac\advfrac6) \bgoodsize$ and 
\[
||\VUempprob-\targetdis||_\cF \le \cO\Paren{\advfrac \sqrt{\frac{\log (1/\advfrac)}{\bsize}}}.
\] 
\end{theorem}

\section{Computationally efficient algorithm for $\cF_k$ distance}\label{sec:F_k}

For discrete domains $\Samplespace=[\ell]$ and $\cF' =2^{\Samplespace}$, where properties~\ref{con}, \ref{con2}, and \ref{con3} hold for all subsets of $D\in \cF'$,~\cite{jain2019robust} derived a method that finds high corruption subsets in $\cF'$ in time polynomial in the domain size $\ell$. Then instead of brute force search over all $2^{[\ell]}$ subsets as in algorithm~\ref{alg3}, they found the subsets with high corruption score efficiently and use the Batch Deletion procedure for these subsets. This lead to a computationally efficient algorithm for learning discrete distributions $p$.

To obtain a computationally efficient algorithm for learning in $\cF_k$ distance over $\Samplespace=\reals$, and derive Theorem~\ref{th:main2}, we first reduce this problem to that
of robust learning distributions over discrete domains in total variation distance and use the algorithm in~\cite{jain2019robust}.

For $\ell>0$, let $\cI_{\ell}$ be the collection of all 
interval partitions $I\triangleq\{I_1\upto I_\ell\}$ of $\reals$.
For $I\in\cI_\ell$, let ${I}^{-1}: \mathbb{R}\rightarrow [\ell]$ map any $x\in\reals$ to the unique $j$ such that $x\in I_j$. 
The mapping $I^{-1}$ converts every continuous distribution $q$ over $\Samplespace = \reals$ 
to the discrete distribution $q^I$ over $\Samplespace=[\ell]$, where $q^I(j)=q(I_j)$ for each $j\in[\ell]$.
Given samples from $q$ the mapping $I^{-1}$
can be used to simulate samples from the distribution $q^I$.

For a subset $D\subseteq [\ell]$ and a partition $I\in \cI_\ell$, let
\[
S_{D}^I
= \cup_{j\in D} I_j,
\]
be the union of $I$ intervals corresponding to elements of $D$.
It follows that for any $I\in \cI_\ell$, distribution $q$ over $\reals$, and $D\subseteq [\ell]$,
\[
q(S_D^I) = q^I(D).
\]

For $I\in \cI_\ell$, define the collection of intervals
\[
\mathcal S(I) \triangleq \{ S_D^I: D \in 2^{[\ell]}\}
\]
to be the family of all possible unions of intervals in $I$. Observe that $\forall\, I\in \cI_\ell$
\[
\mathcal S(I)\subseteq \cF_\ell.
\]

The next theorem describes a simple modification of a
polynomial-time algorithm in~\cite{jain2019robust}, 
that for any $I\in \cI$ returns a sub-collection $\allbatches^*$
of batches whose empirical distribution estimates $\targetdis$ to a small $\cS(I)$-distance. 

\begin{theorem}\label{th:ahsss}
If Properties~\ref{con},~\ref{con2}, and~\ref{con3} in Lemma~\ref{lem:prophold} hold for $\cF'=\cF_\ell$, then for any given partition $I\in \cI_\ell$, there is an algorithm that runs in time polynomial in partition size $\ell$, number of batches $m$, and batch-size $n$, and with probability $\ge 1-e^{-O(\advfrac\btotal)}$ returns a sub-collection of batches $\allbatches^*\subseteq\allbatches$ such that $\allbatches^*\cap\bgood \ge (1-\advfrac/6)\bgoodsize$ and
\[
||\targetdis-\bar\targetdis_{\allbatches^*} ||_{\cS(I)}\le 100\advfrac\sqrt{\frac{\ln(1/\advfrac)}{\bsize}}.
\]
\end{theorem}
\begin{proof}
Suppose Properties~\ref{con}--\ref{con3} hold for all subsets in $\cF_\ell$.
Since $\cF_\ell\supseteq \mathcal S(I)$ for all $I\in\cI_\ell$, these properties hold for all subsets in $\mathcal S(I)$.
For any partition $I\in\cI_\ell$, the one-to-one correspondence $I^{-1}$ maps
samples in $\reals$ to $[\ell]$, and subsets in $\cS(I)$ to subsets in $2^{[\ell]}$. This implies that the three properties hold also for the transformed distribution $p^I$ and the batches of discretized samples for all subsets of $2^{[\ell]}$. 

Recall that $\VUempprob$ denotes the empirical distribution induced by a sub-collection $\Bsc$, therefore ${\VUempprob}^I$ denotes the empirical distribution induced by a sub-collection $\Bsc$ over the transformed domain $[\ell]$.

Since these properties hold, Theorem~9 in~\cite{jain2019robust} implies that algorithm~2 therein runs in time polynomial in the domain size $\ell$, the number of batches $m$, and the batch-size $n$, and with probability $\ge 1-e^{-O(\advfrac\btotal)}$ returns a sub-collection of batches $\allbatches^*\subseteq\allbatches$ such that $\allbatches^*\cap\bgood \ge (1-\advfrac/6)\bgoodsize$ and
\[
||\targetdis^I-\bar\targetdis^I_{\allbatches^*} ||_{TV}\le 100\advfrac\sqrt{\frac{\ln(1/\advfrac)}{\bsize}}.
\]

Next we show that a pair of distributions $q_1$ and $q_2$ over the reals is close in $\cS(I)$-distance iff $q_1^I$ and $q_2^I$ are close in total variation distance. 
For every distribution pair $q_1,q_2$ over $\reals$, 
\begin{align*}
 ||q_1-q_2||_{\cS(I)} &= \max_{S\in \cS(I)} |q_1(S)-q_2(S)|\nonumber\\
 &= \max_{S_D^I\in \cS(I)} |q_1(S_D^I)-q_2(S_D^I)|\nonumber\\
 &= \max_{D\in 2^{[\ell]}} |q_1^I(D)-q_2^I(D)|\nonumber\\
&= ||q_1^I-q_2^I||_{TV}.
\end{align*}

Therefore the empirical distribution of the sub-collection $\allbatches^*$ of samples over the original domain $\reals$ estimates $\targetdis$ in ${\cS(I)}$-distance,
\[
||\targetdis-\bar\targetdis_{\allbatches^*} ||_{\cS(I)}\le 100\advfrac\sqrt{\frac{\ln(1/\advfrac)}{\bsize}}.\qedhere
\]
\end{proof}

Next, we construct $I^*\in \cI_\ell$ such that $\cS(I^*)$ is a $\frac{2k}\ell$-cover of $\cF_k$ w.r.t. the empirical measure $\bar\targetdis_{\allbatches}$. %Theorem~\ref{th:relate} then implies...

Recall that $\allbatches$ is a collection of $m$ batches and each batch has $n$ samples. Let $s=n\cdot m$ and let $x^s=x_1,x_2\upto x_s\in\reals$ be the samples of $\allbatches$ arranged in non-decreasing order.
And recall that the points $x^s$ induce an empirical measure $\bar\targetdis_{\allbatches}$ over $\reals$, where
for $S\subseteq\reals$,
\[
\bar\targetdis_{\allbatches}(S)=|\sets{i: x_i\in S}|/s.
\]

Let $\Delta\ed\frac s\ell$, and for simplicity assume that it is an integer.
Construct the $\ell$-partition
$I^* \ed \{I^*_1\upto I^*_\ell\}$ of $\reals$, where
\[
I^*_j
\triangleq
\begin{cases}
(-\infty,x_{\Delta}] & j=1,\\
(x_{(j-1)\Delta},x_{j\Delta}] & 2\le j<\ell,\\
(x_{t-\Delta},\infty) & j=\ell.
\end{cases}
\]

We show that $\mathcal S(I^*)$ is
an $2k/\ell-$cover of $\cF _k$  w.r.t. the empirical measure $\bar\targetdis_{\allbatches}$ of points $x_1^s$. 

\begin{lemma}
For any $k$, and $\ell$, $\cS(I^*)$ is an $\frac{2k}\ell$-cover of $\cF _k$ w.r.t. $\bar\targetdis_{\allbatches}$.
\end{lemma}

\begin{proof}
Any set $S\in \cF _k$ is a union of $k$ real intervals $I_1\cup I_2 \cup\ldots\cup I_k$.
Let $S^*\subseteq\reals$ be the union of all $P_j$-intervals that  are fully
contained in one of the intervals $I_1\upto I_k$.
By definition, $S^*\in \mathcal S(I^*)$, and we show that $\bar\targetdis_{\allbatches}(S\triangle S^*) \le 2k/\ell$. 
By construction, $S^*\subseteq S$, hence,
\[
\bar\targetdis_{\allbatches}(S\triangle S^*) 
=
\bar\targetdis_{\allbatches}(S\setminus S^*)
=
\sum_{j=1}^k\bar\targetdis_{\allbatches}(I_j\setminus S^*) 
=
\sum_{j=1}^k\frac{|\{x_i\in I_j\setminus S^* \}|}s
\le
\sum_{j=1}^k 2\cdot \frac{\Delta}{s} 
=\frac{2k}{\ell},
\]
\ignore{
\begin{align*}
\bar \mu_{x^t}(S\triangle S^*) &= \bar \mu_{x^t}(S\setminus S^*)\\
&= \sum_{j\in [{k}]}\bar \mu_{x^t}(I_j\setminus S^*) 
= \sum_{j\in [{k}]} |\{x_i\in I_j\setminus S^* \}|/t\\
&\overset{\text{(a)}}\le \sum_{j\in [k]} 2\cdot \frac{\epsilon t}{2k t} \overset{\text{(b)}}
=\frac{\epsilon {k^\prime} }{k }
\le \epsilon,
\end{align*}
}
where the inequality follows as each $I_j\setminus S^*$ contains at most $\Delta$ points and the left and right.
\end{proof}

Next choose $\ell= \frac{2k\sqrt{n}}{\advfrac}$ then the lemma implies that the corresponding $\cS(I^*)$ is an $\frac{\epsilon}{\sqrt{n}}$ cover. Combining Theorems~\ref{th:relate} and~\ref{th:ahsss}, and the Lemma, we get the following theorem that implies learning in $\cF_k$ distance.
\begin{theorem}[Theorem~\ref{th:main2} restated]
For any given $\advfrac\le 0.4$, $\delta>0$, $\bsize$, $k>0$, and
$\btotal\ge \cO\Paren{\frac{k \log(\bsize/\advfrac) +\log 1/\delta}{\advfrac^3 }\cdot\sqrt{\bsize}}$,
there is an algorithm that runs in time polynomial in all parameters, and with probability $\ge 1- \delta$ returns a sub-collection of batches $\Bsc$ such that $|\Bsc\cap\bgood| \ge (1 -\frac\advfrac6) \bgoodsize$ and 
\[
||\VUempprob-\targetdis||_{\cF_k} \le \cO\Paren{\advfrac \sqrt{\frac{\log (1/\advfrac)}{\bsize}}}.
\] 
\end{theorem}

\section*{Acknowledgements}
We thank Vaishakh Ravindrakumar, Yi Hao and Jerry Li for helpful discussions and comments in the prepration of this manuscript.

We are grateful to the National Science
Foundation (NSF) for supporting this work through grants
CIF-1564355 and CIF-1619448.

\bibliographystyle{alpha}
\bibliography{ref}

\appendix
\section{Properties of the Collection of Good Batches}\label{sec:goodprop}

\begin{lemma}\label{lem:impofrd}
Let $\cF$ be a VC family of subsets of $\Samplespace$. Then for any $\delta> 0$ and $\bgoodsize \ge \cO(\frac{V_{\cF}\log(\bsize/\advfrac)+\log(1/\delta)}{\advfrac^2})$, with probability $\ge1-\delta$,
\begin{align*}
    \sup_{S,S'\in\cF } \max \Big\{\frac{\bar\targetdis_{\bgood} (S\triangle S') - \targetdis(S\triangle S')}{\sqrt{\bar\targetdis_{\bgood} (S\triangle S') }},\frac{\targetdis(S\triangle S')-\bar \targetdis_{\bgood} (S\triangle S') }{\sqrt{\targetdis(S\triangle S') }}\Big\}\le \frac{\advfrac}{\sqrt{\bsize}}.
\end{align*}
\end{lemma}
\begin{proof}
Consider the collection of symmetric differences of subsets in $\cF $,
\[
\cF _{\triangle} \triangleq \{ S\triangle S' : S, S' \in \cF \}.
\]
The next auxiliary lemma bounds the shatter coefficient of $\cF _\triangle$.
\begin{lemma}\label{lem:shatter_int}
For $t \ge V_{\cF }$, $S_{\cF _{\triangle}}(t)\le \big(\frac{t\,e}{V_{\cF }}\big)^{2V_{\cF }}$. 
\end{lemma}
\begin{proof}
For $t\ge V_{\cF }$ and $x_1,x_2,..,x_t\in\Samplespace$, 
let
\[
\cF (x_1^t) =  \{\{x_1,x_2,..,x_t\}\cap S : S\in\cF \}.
\]
Note that $S_{\cF }(t) = \max_{x_1\upto x_t} |\cF (x_1^t)| $.

From the definition of shatter coefficient $|\cF (x_1^t)| \le S_{\cF }(t)$. Then 
\[
|\cF _\triangle(x_1^t)| =  |\{ \{x_1\upto x_t\}\triangle \{x_1^\prime\upto x_t^\prime\}: S,S^\prime\in\cF (x_1^t)\}|\le (S_{\cF }(t))^2\le \big(\frac{t\,e}{V_{\cF }}\big)^{2V_{\cF }}.\hfill\qedhere
\]
\end{proof}
Recall that the sub-collection of good batches has $n\bgoodsize$ samples. Then applying Theorem~\ref{th:vcrel} for family of subsets $\cF_\triangle$, and using Lemma~\ref{lem:shatter_int}, for $\bgoodsize \ge \cO(\frac{V_{\cF}\log(\bsize/\advfrac)+\log(1/\delta)}{\advfrac^2})$, with probability $\ge 1-\delta$,
\[
\sup_{S\in\cF_\triangle} \max \Big\{\frac{\bar\targetdis_{\bgood} (S) - \subsetprobtarget}{\sqrt{\bar\targetdis_{\bgood} (S) }},\sup_{S\in\cF }\frac{\subsetprobtarget-\bar \targetdis_{\bgood} (S) }{\sqrt{\subsetprobtarget }}\Big\}\le \frac{\advfrac}{\sqrt{\bsize}}.\hfill\qedhere
\]
\end{proof}

\subsection{Proof of Lemma~\ref{lem:prophold}}
We prove the Lemma without the constants stated in the properties here for simplicity of the presentation. The constant stated can be obtained with a more careful calculations.

In this section, we show that the properties~\ref{con}-\ref{con3} hold when the family $\cF'$ has a finite VC-dimension. 

The proof of a similar Lemma~\cite{jain2019robust} establish that for to show that the properties~\ref{con}-~\ref{con3} can be shown to hold for a subset $S$ if the following conditions are satisfied for $S$.

For any $\advfrac\in (0,0.4]$, 
\begin{enumerate}
\item For all ${\Gsc}\subseteq\bgood$, such that $|{\Gsc}| \ge (1-\advfrac/6)\bgoodsize$
\begin{align}
    &|\bar \targetdis_{{\Gsc}} (\alphsubset) - \targetdis(\alphsubset) | \le  \cO\Paren{\advfrac \sqrt{\frac{\ln (1/\advfrac) } {\bsize}}},\label{eq:lem1}\\
    &\Big|\frac{1}{|{\Gsc}|}\sum_{b\in {\Gsc}} (\bar \mu_b(\alphsubset)- \targetdis(\alphsubset))^2 - \var{\targetdis(\alphsubset^\prime)}\Big|\le  \cO\Paren{\frac{ \advfrac\ln (\frac{1}\advfrac) } {\bsize}} \label{eq:lem2}.
\end{align}
\item
\begin{align}
    &\big|\big\{b\in \bgood: |\bar \mu_b(\alphsubset)- \targetdis(\alphsubset)|\ge  \cO\Paren{\sqrt{\frac{\ln (1/\advfrac)}\bsize}} \big\}\big|\le O(1)\cdot\bgoodsize\advfrac.\label{eq:lem3}
\end{align}
\item For all ${\Gsc}\subseteq\bgood$, such that $|{\Gsc}| \le \cO(\advfrac)\bgoodsize$
\begin{align}
    &\sum_{b\in \bgood^{d}(\alphsubset,\epsilon)} (\bempprob- \subsetprobtarget)^2 < \cO\Paren{\advfrac\bgoodsize {\frac{ \ln (1/\advfrac) } {\bsize}}},\label{eq:lem4}
\end{align}
\end{enumerate}

They also showed that above conditions hold for all subsets in a fixed finite collection of subsets $\cC$, with probability $\ge 1-\delta$, if 
$\bgoodsize\ge O(\frac{\log |\cC|+\log 1/\delta}{\advfrac^2\ln (1/\advfrac)})$.

But this doesn't give the result for subsets in a general VC class $\cF'$ as it may have uncountable subsets.

From Corollary~\ref{cor:intcovbou}, there exist a minimal-self $\epsilon$-cover $C^*$ of $\cF'$ w.r.t. distribution $\targetdis$ of size $\cO\Paren{V_{\cF'}  (\frac{8e}{\epsilon})^{ V_{\cF'} }}$. 
Fix $\epsilon= \cO(\frac{\advfrac^2}{{\bsize}})$. 

Therefore, for $\bgoodsize\ge O(\frac{V_{\cF'}\log(n/\advfrac)+\log 1/\delta}{\advfrac^2\ln (1/\advfrac)})$, the above properties hold for all subsets in $C^*$.

To complete the proof, we show if the above conditions hold for all subsets in $\cC^*$, they also hold for all subsets in $\cF'$. 
For subset $S\in \cF' $ choose $ S'\in \mathcal{C}^*$ such that $p(S\triangle S')\le \epsilon$. Existence of such a subset $S'\in\mathcal{C}^*$ is guaranteed for all $S\in \cF'$ as $\mathcal{C}^*$ is an $\epsilon-$cover  w.r.t. $p$. 

Note that for any subset $\alphsubset,S'\in\cF'$ with $\targetdis(S\triangle S') \le \cO(\frac{\advfrac^2}{{\bsize}})$, Lemma~\ref{lem:impofrd} implies
\begin{equation}
\label{eq:reldev}
\bar \targetdis_{\bgood} (S\triangle S')
\le
 \cO(\frac{ \advfrac^2}{{\bsize}})=\cO(\epsilon).
\end{equation}

Then for any batch $b\in B$
\begin{align*}
    \bempprob-\subsetprobtarget&=  \Big(\bar \mu_b(S')+\bar \mu_b(S\setminus S')-\bar \mu_b(S'\setminus S)\Big)-\Big(\targetdis(S')+ \targetdis(S\setminus S')-\targetdis(S'\setminus S)\Big)\\
    &=  \Big(\bar \mu_b(S')- \targetdis(S')\Big) +\Big(\bar \mu_b(S\setminus S')-\bar \mu_b(S'\setminus S)\Big)-\Big(\targetdis(S\setminus S')-\targetdis(S'\setminus S)\Big)
\end{align*}
From the above equation we get
\begin{align}
    \Big|\Big(\bempprob-\subsetprobtarget\Big)- \Big(\bar \mu_b(S')- \targetdis(S')\Big)\Big|&\le\bar \mu_b(S\setminus S')+\bar \mu_b(S'\setminus S)+ \targetdis(S\setminus S')+\targetdis(S'\setminus S)\nonumber\\
    & = \bar \mu_b(S\triangle S')+\targetdis(S\triangle S') \nonumber\\
    &\le \bar \mu_b(S\triangle S')+\cO(\epsilon).\label{eq:absdiffcov}
\end{align}
Next we generalise condition~\eqref{eq:lem1} to any subset $S\in\cF'$.
\begin{align*}
    |\bar \targetdis_{{\Gsc}} (S) - \subsetprobtarget | &= \Big|
     \frac{1}{|{\Gsc}|}\sum_{b\in {\Gsc}}\bempprob - \subsetprobtarget \Big|= \Big|
     \frac{1}{|{\Gsc}|}\sum_{b\in {\Gsc}}\Big(\bempprob - \subsetprobtarget\Big) \Big|\\
     &\overset{\text{(a)}}\le \Big|
     \frac{1}{|{\Gsc}|}\sum_{b\in {\Gsc}}\Big(\bar \mu_b(S')- \targetdis(S')\Big) \Big|+\Big|
     \frac{1}{|{\Gsc}|}\sum_{b\in {\Gsc}}\Big(\bar \mu_b(S\triangle S')+\cO(\epsilon)\Big) \Big| \\
     &\le \Big|
     \frac{1}{|{\Gsc}|}\sum_{b\in {\Gsc}}\bar \mu_b(S')- \targetdis(S') \Big|+\Big|
     \frac{1}{|{\Gsc}|}\sum_{b\in \bgood}\bar \mu_b(S\triangle S') \Big| +\cO(\epsilon)\\
     &\le |\bar \targetdis_{{\Gsc}}(S')- \targetdis(S')|+
     \frac{\bgoodsize}{|{\Gsc}|}\bar \targetdis_\bgood(S\triangle S') +\cO(\epsilon)\\
      &\le \cO\Paren{\advfrac \sqrt{\frac{\ln (1/\advfrac) } {\bsize}}}+
     \frac{1}{(1-\advfrac/6)}\cdot\cO(\epsilon)+\cO(\epsilon)\\
     &\le \cO\Paren{\advfrac \sqrt{\frac{\ln (1/\advfrac) } {\bsize}}},
\end{align*}
here (a) uses~\eqref{eq:absdiffcov}.

Next we generalise condition~\eqref{eq:lem2} to subsets $S\in\cF'$. From equation~\eqref{eq:absdiffcov} we get
\begin{align*}
    &(\bempprob -   \subsetprobtarget)^2 \le 
    \Big(|\bar \mu_{b}(S')- \targetdis(S')|+(\bar\mu_b(S\triangle S')+\cO(\epsilon))\Big)^2\\
    &=(\bar \mu_{b}(S')- \targetdis(S'))^2+ 2|\bar \mu_{b}(S')- \targetdis(S')|(\bar\mu_b(S\triangle S')+\cO(\epsilon))+(\bar\mu_b(S\triangle S')+\cO(\epsilon))^2.
\end{align*}
Therefore,
\begin{align*}
    &\sum_{b\in {\Gsc}} (\bempprob -   \subsetprobtarget)^2-\sum_{b\in {\Gsc}} (\bar \mu_{b}(S')- \targetdis(S'))^2 \\
    &\le \sum_{b\in {\Gsc}}2|\bar \mu_{b}(S')- \targetdis(S')|(\bar\mu_b(S\triangle S')+\cO(\epsilon))+\sum_{b\in {\Gsc}} (\bar\mu_b(S\triangle S')+\cO(\epsilon))^2\\
    &\le 2\sqrt{\sum_{b\in {\Gsc}} (\bar \mu_{b}(S')- \targetdis(S'))^2}\sqrt{\sum_{b\in {\Gsc}} (\bar\mu_b(S\triangle S')+\cO(\epsilon))^2}+\sum_{b\in {\Gsc}} (\bar\mu_b(S\triangle S')+\cO(\epsilon))^2, 
\end{align*}
here the last inequality follows from Cauchy-Schwarz inequality.  Next, we bound the last terms on the right in above expression.
\begin{align*}
\sum_{b\in {\Gsc}} (\bar\mu_b(S\triangle S')+\cO(\epsilon))^2
&\le \sum_{b\in {\Gsc}} (\bar\mu_b(S\triangle S')+\cO(\epsilon)) (1+\cO(\epsilon)) \\
&\le 2\cdot \left( |{\Gsc}|\cO(\epsilon) + \sum_{b\in \bgood} (\bar\mu_b(S\triangle S')\right)\\
&\le 2 |{\Gsc}| \left( \cO(\epsilon) + \frac{\bgoodsize}{|{\Gsc}|} \bar\targetdis_\bgood(S\triangle S') \right)\\
&\le |{\Gsc}|\cO(\epsilon).
\end{align*}
Also,
\begin{align*}
  \sum_{b\in {\Gsc}} (\bar \mu_{b}(S')- \targetdis(S'))^2 &\le |{\Gsc}| \left( \cO\Paren{\frac{ \advfrac\ln (\frac{1}\advfrac) } {\bsize}}+ \var{\targetdis(S')}\right)\\
  &\le |{\Gsc}| \cO\left(\frac {1}{\bsize} \right),
\end{align*}
here we used equation~\eqref{eq:fineq} and the fact that $\advfrac\ln (e/\advfrac)=\cO(1)$.
Combining the above three equations we get
\begin{align*}
    &\sum_{b\in {\Gsc}} (\bempprob -   \subsetprobtarget)^2-\sum_{b\in {\Gsc}} (\bar \mu_{b}(S')- \targetdis(S'))^2 \\
    &\le 2\sqrt{|{\Gsc}| \cO\left(\frac {1}{\bsize} \right)}\sqrt{|{\Gsc}| \cO(\epsilon)}+|{\Gsc}| \cO(\epsilon)
    < |{\Gsc}| \cO\left(\sqrt{ \frac {\epsilon}{\bsize}}\right). 
\end{align*}
Similarly, one can prove the other direction of the inequality to get the following
\begin{align*}
    &\Big|\sum_{b\in {\Gsc}} (\bempprob -   \subsetprobtarget)^2-\sum_{b\in {\Gsc}} (\bar \mu_{b}(S')- \targetdis(S'))^2 \Big|
    < |{\Gsc}| \cO\left(\sqrt{ \frac {\epsilon}{\bsize}}\right). 
\end{align*}
And from~\eqref{eq:fineq} we get
\begin{align*}
    |\var{\subsetprobtarget}- \var{\targetdis(S')}| \le \frac{|{\subsetprobtarget}-\targetdis(S')|}{\bsize}\le \frac{|\targetdis(\alphsubset\triangle S')|}{\bsize} \le  \cO\left({ \frac {\epsilon}{\bsize}}\right).
\end{align*}
From the above two equations we get
\begin{align}
&\Big|\frac{1}{|{\Gsc}|}\sum_{b\in {\Gsc}} (\bempprob- \subsetprobtarget)^2 - \var{\subsetprobtarget}\Big|\nonumber\\
&\le 
\Big|\frac{1}{|{\Gsc}|}\sum_{b\in {\Gsc}} (\bar \mu_{b}(S')- p(S'))^2 - \var{p(S')}\Big|+ \cO\left(\sqrt{ \frac {\epsilon}{\bsize}}\right)+\cO\left({ \frac {\epsilon}{\bsize}}\right)\nonumber\\
&\overset{\text{(a)}}\le \cO\Paren{\frac{ \advfrac\ln (\frac{1}\advfrac) } {\bsize}}+ \cO\left(\sqrt{ \frac {\epsilon}{\bsize}}\right)+\cO\left({ \frac {\epsilon}{\bsize}}\right)\nonumber\\
&\overset{\text{(b)}}\le  \cO\Paren{\frac{ \advfrac\ln (\frac{1}\advfrac) } {\bsize}},
\end{align}
here inequality (a) uses equation~\eqref{eq:lem2}, (b) uses $\epsilon \le \cO\Paren{\frac{ \advfrac^2\ln (\frac{1}\advfrac) } {\bsize}}$. 

This completes the proof of the extension of condition~\eqref{eq:lem2} to subsets $S\in\cF'$ and in a similar fashion condition~\eqref{eq:lem4} can be extended.

Next, we extend condition~\eqref{eq:lem3} to subsets $S\in\cF'$.
\begin{align}
&\big|\big\{b\in \bgood: |\bempprob-\subsetprobtarget|\ge t \big\}\big| \nonumber\\
&\le \big|\big\{b\in \bgood: |\bar \mu_b(S')- p(S')|+\bar \mu_b(S\triangle S')+\cO(\epsilon)
\ge t \big\}\big|\nonumber\\
&\le \big|\big\{b\in \bgood: |\bar \mu_b(S')- p(S')|
\ge  \frac 2 3\cdot t \big\}\big|+\big|\big\{b\in \bgood: \bar \mu_b(S\triangle S')
\ge \frac t 3 -\cO(\epsilon) \big\}\big|\nonumber\\
&\le \big|\big\{b\in \bgood: |\bar \mu_b(S')- p(S')|
\ge  \frac 2 3\cdot t \big\}\big|+\frac{\sum_{b\in \bgood}\bar \mu_b(S\triangle S')}{\frac t 3 -\cO(\epsilon)}\nonumber\\
&\le \big|\big\{b\in \bgood: |\bar \mu_b(S')- p(S')|
\ge \frac 2 3\cdot t \big\}\big|+\bgoodsize\frac{\bar p_\bgood(S\triangle S')}{\frac t 3 -\cO(\epsilon)}\nonumber\\
&\le \big|\big\{b\in \bgood: |\bar \mu_b(S')- p(S')|
\ge  \frac 2 3\cdot t \big\}\big|+\bgoodsize\frac{\cO(\epsilon)}{\frac t 3 -\cO(\epsilon)}\nonumber\\
&\le \big|\big\{b\in \bgood: |\bar \mu_b(S')- p(S')|
\ge  \frac 2 3\cdot t \big\}\big|+\bgoodsize\frac{\cO(\epsilon)}{t-\cO(\epsilon)}.\label{eq:asd}
\end{align}
Choosing $t=\cO\Paren{\sqrt{\frac{\ln (1/\advfrac)}\bsize}}$ in the above equation extends condition~\eqref{eq:lem3} to subsets $S\in\cF'$.

\section{Proof of Lemma~\ref{lem:relloss}}
\begin{proof}
\begin{align}
&r_p(h^*(q))-r^*_p(\mathcal{H})\nonumber \\
&= 
r_p(h^*(q))-r_p(h^*(p))\nonumber \\
&= r_p(h^*(q))-r_q(h^*(q))+r_q(h^*(q))-r_q(h^*(p))+r_q(h^*(p))-r_p(h^*(p))\nonumber \\
&\le r_q(h^*(q))-r_q(h^*(p))+2\sup_{h\in\mathcal{H}}|r_q(h)-r_p(h)|\nonumber  \\
&\le 2\sup_{h\in\mathcal{H}}|r_q(h)-r_p(h)|\nonumber  \\
&\le  4 ||p-q||_{\cF_{\mathcal{H}}},\nonumber
\end{align}
here the last inequality uses~\eqref{eq:apw}.
\end{proof}

%\appendix
%\appendix
%\section{Preliminaries} 

\end{document}